\documentclass{article}

\usepackage{microtype}
\usepackage{graphicx}
\usepackage{subfigure}
\usepackage{booktabs}
\usepackage{hyperref}

\usepackage[accepted]{icml2023}

\usepackage{amsmath}
\usepackage{amssymb}
\usepackage{mathtools}
\usepackage{amsthm}

\usepackage{graphicx}
\usepackage{subfigure}
\usepackage[acronym]{glossaries}
\usepackage{float}
\usepackage{caption}
\usepackage{wrapfig}

\usepackage[capitalize,noabbrev]{cleveref}

\theoremstyle{plain}
\newtheorem{theorem}{Theorem}[section]
\newtheorem{proposition}[theorem]{Proposition}

\theoremstyle{definition}

\theoremstyle{remark}

\usepackage[textsize=tiny]{todonotes}

\glsdisablehyper

\newacronym{rl}{RL}{Reinforcement Learning}
\newacronym{drl}{deep RL}{Deep Reinforcement Learning}
\newacronym{dl}{DL}{Deep Learning}
\newacronym{mdp}{MDP}{Markov Decision Process}
\newacronym{sr}{SR}{successor representation}
\newacronym{sf}{SF}{Successor Features}
\newacronym{td}{TD}{Temporal Difference}
\newacronym{gpi}{GPI}{Generalized Policy Improvement}
\newacronym{gvf}{GVFs}{Generalized Value Functions}
\newacronym{usfa}{USFAs}{Universal Successor Feature Approximators}
\newacronym{uvf}{UVFs}{Universal Value Functions}
\newacronym{ts}{TS}{Thomson Sampling}
\newacronym{ecdp}{ECDP}{Extended Controlled Markov Process}
\newacronym{et}{ET}{Eligibility Traces}
\newacronym{gru}{GRU}{Gated Recurrent Unit}
\newacronym{mlp}{MLP}{Multi-Layer Perceptron}
\newacronym{cmp}{CMP}{Controlled Markov Process}
\newacronym{rnn}{RNN}{Recurrent Neural Network}
\newacronym{sm}{SM}{Successor Measures}
\newacronym{msve}{MSVE}{Maximum State-Visitation Entropy}
\newacronym{dpg}{DPG}{Deterministic Policy Gradient}

\newcommand{\alg}{$\eta$$\psi$-Learning}

\icmltitlerunning{Maximum State Entropy Exploration using Predecessor and Successor Representations}

\begin{document}

\twocolumn[
    \icmltitle{Maximum State Entropy Exploration using Predecessor\\and Successor Representations}          
    \icmlsetsymbol{equal}{*}
    \begin{icmlauthorlist}
    \icmlauthor{Arnav Kumar Jain}{mila,udem}
    \icmlauthor{Lucas Lehnert}{fair}
    \icmlauthor{Irina Rish}{mila,udem}
    \icmlauthor{Glen Berseth}{mila,udem}
    \end{icmlauthorlist}
    \icmlaffiliation{mila}{Mila---Quebec AI Institute}
    \icmlaffiliation{udem}{University de Montreal}
    \icmlaffiliation{fair}{Meta AI, FAIR}
    \icmlcorrespondingauthor{Arnav Kumar Jain}{arnav-kumar.jain@mila.quebec}
    \icmlkeywords{Machine Learning, ICML}
    \vskip 0.3in
]
\printAffiliationsAndNotice{}

\begin{abstract}
Animals have a developed ability to explore that aids them in important tasks such as locating food, exploring for shelter, and finding misplaced items.
These exploration skills necessarily track where they have been so that they can plan for finding items with relative efficiency.
Contemporary exploration algorithms often learn a less efficient exploration strategy because they either condition only on the current state or simply rely on making random open-loop exploratory moves.
In this work, we propose \alg{}, a method to learn efficient exploratory policies by conditioning on past episodic experience to make the next exploratory move.
Specifically, \alg{} learns an exploration policy that maximizes the entropy of the state visitation distribution of a single trajectory.
Furthermore, we demonstrate how variants of the predecessor representation and successor representations can be combined to predict the state visitation entropy. 
Our experiments demonstrate the efficacy of \alg{} to strategically explore the environment and maximize the state coverage with limited samples. 
\end{abstract}

\begin{figure*}[t]
\centering
\subfigure[]{
\includegraphics[width=.2\textwidth]{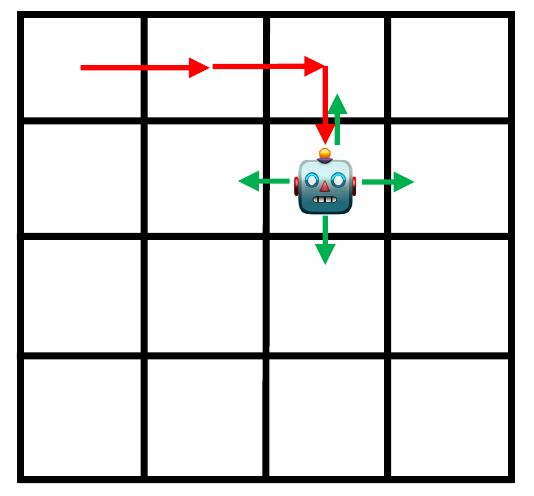}
\label{fig:motivation_trace} 
}
\subfigure[]{
\includegraphics[width=.2\textwidth]{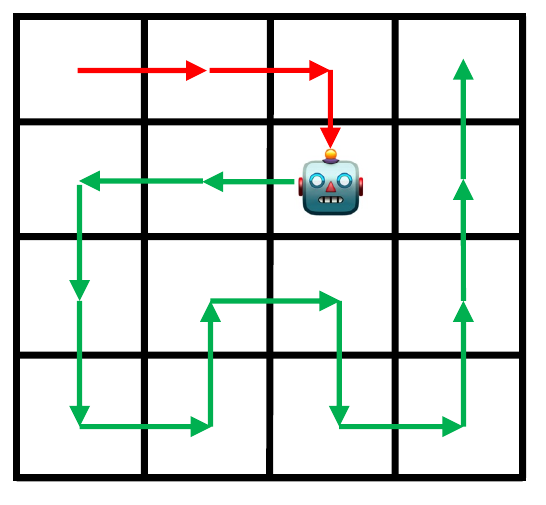}
\label{fig:motivation_traj1} 
}
\subfigure[]{
\includegraphics[width=.2\textwidth]{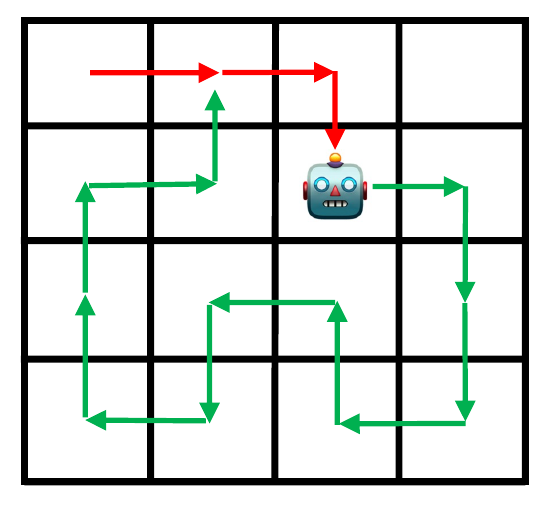}
\label{fig:motivation_traj2}
}
\subfigure[]{
\includegraphics[width=.207\textwidth]{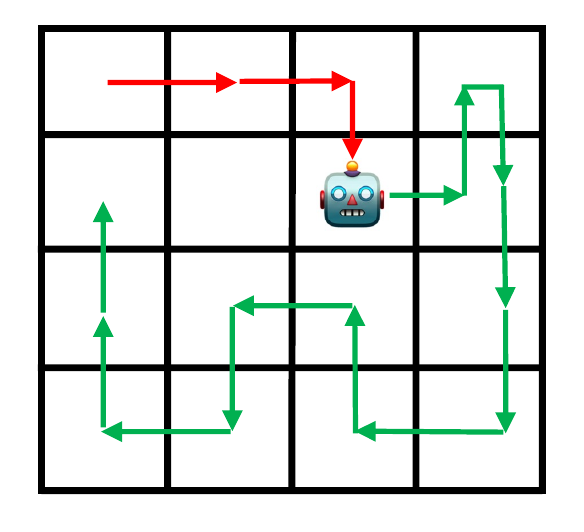}
\label{fig:motivation_traj3}
}
\caption{
Consider a 4$\times$4 grid-world for illustration. (a) Agent starts at the top left corner and takes a few actions(red arrows show the trace), (b) optimal trajectory covering the grid (green arrows), (c) sub-optimal trajectory where agent visits a previously observed state in the last step, (d) another sub-optimal trajectory an observed state is visited at an earlier step.
}
\label{fig:motivation}
\end{figure*}

\section{Introduction}
Animals and humans are very efficient at exploration compared to their data-hungry algorithms counterparts~\citep{litman2005curiosity,kidd2015psychology, schmidhuber1991curiousmodel, schmidhuber2009simple, clark2018nice, ecoffet2019go,vinyals2017starcraft,schmidhuber2010formal}.
For instance, when misplacing an item, a person will methodically search through the environment to locate the lost item. 
To efficiently explore, an intelligent agent must consider past interactions to decide what to explore next and avoid re-visiting previously encountered locations to find rewarding states as fast as possible.
Consequently, the agent needs to reason over potentially long interaction sequences, a space that grows exponentially with the sequence length.
Here, assuming that all the information the agent needs to act is contained in the current state is impossible~\citep{mutti2022maximumentropyexploration}.

In this paper, we present \alg{}, an algorithm to learn an exploration policy that methodically searches through a task.
This is accomplished by maximizing the entropy of the state visitation distribution of a single finite-length trajectory.
This focus on optimizing the state visitation distribution of a single trajectory distinguishes our approach from prior work on exploration methods that maximize the entropy of the state visitation distribution~\cite{pmlr-v97-hazan19a, tarbouriech2019active,lee2019efficient,mutti2020intrinsically,guo2021geometric}.
For example~\citet{pmlr-v97-hazan19a} focus on learning a Markovian policy---a decision-making strategy that is only conditioned on the current state and does not consider which states have been explored before.
A Markovian policy constrains the agent in its ability to express different exploration policies and typically results in randomizing at uncertain states to maximize the state entropy objective.
While this will lead to uniformly covering the state space in the limit, such behaviors are not favorable for real-world tasks where the agent needs to maximize state coverage with limited number of interactions.  

\autoref{fig:motivation} presents a didactic example to illustrate how an intelligent agent can learn to efficiently explore a $4 \times 4$ grid.
In this example, the agent transitions between different grid cells by selecting one of four actions: \emph{up}, \emph{down}, \emph{left}, and \emph{right}.
To explore optimally, the agent would select actions that maximize the entropy of the visited state distribution of the entire trajectory.
Suppose the agent started its trajectory in the top left corner of the grid (shown in Figure~\ref{fig:motivation_trace}) and has moved to the right twice and made one downward step (indicated by red arrows). 
At this point, the agent has to decide between one of the four actions to further explore the grid.
For example, it could move \emph{left} and follow the green trajectory as outlined in Figure~\ref{fig:motivation_traj1}.
This path would be optimal in this example because every state is visited exactly once and not multiple times.
However, the \emph{top} action would lead to a sub-optimal trajectory as the agent would visit the previous state.
To mitigate sub-optimal exploration, an intelligent agent must keep track of visited states to avoid revisiting states. 
Although taking the \emph{right} action will lead to a novel state in the next step, the overall behavior will be sub-optimal as the agent will have to visit a state twice to explore the entire grid (depicted in Figures~\ref{fig:motivation_traj2} and \ref{fig:motivation_traj3}). 
This further requires an agent to carefully plan and account for the states that would follow after taking an action. 

In this work, we propose \alg{}, an algorithm to compute an exploration strategy that methodically explores within a single finite-length trajectory---as illustrated in Figure~\ref{fig:motivation_traj1}.
\alg{} maintains two state representations: a predecessor representation~\cite{vanhasselt2021expectedeligibilitytraces,bailey2022predecessor} to encode past state visitation frequencies and a Successor Representation (SR)~\cite{dayan1993improving} to predict future state visitation frequencies.
The two representations are used to evaluate at every time step the decision that leads to covering all states as uniformly as possible.
Specifically, for every potential action the agent can take, the SR is combined with the predecessor representation to predict the state visitation distribution for the current trajectory.
Then, the action that results in the highest entropy of this state visitation distribution is selected for exploration.
Furthermore, this exploration policy can be deterministic and does not randomize to achieve its maximum state entropy objective. 

To summarize, the contributions of this work are as follows:
Firstly, we propose a mechanism to combine successor~\citep{dayan1993improving} and predecessor~\citep{vanhasselt2021expectedeligibilitytraces} representations for maximizing the entropy of the state visitation distribution of a finite-length trajectory. 
To the best of our knowledge, this is the first work using the two representations to optimize the state visitation distribution and learn an efficient exploration policy. 
Secondly, we introduce \alg{}, a method that utilizes the combination of two representations to learn deterministic and non-Markovian exploration policies for the finite-sample regime.
Thirdly, we discuss how \alg{} optimizes the entropy-based objective function for both finite and (uncountably) infinite action spaces. 
In Section~\ref{sec:experiments} we demonstrate through empirical experiments that \alg{} achieves optimal coverage within a single finite-length trajectory. 
Moreover, the visualizations presented in Section~\ref{sec:experiments} demonstrate that \alg{} learns an exploration policy that maneuvers through the state space to efficiently explore a task while minimizing the number of times the same state is revisited.

\section{Related Work}
The domain of exploration in \gls{rl} focuses on discovering an agent's environment via intrinsic motivation to accelerate learning optimal policies.
Many of the existing exploration methods seek novelty by using prediction errors
~\cite{pathak2017icm, burda2019exploration,sekar2020planning,stadie2015incentivizing} or pseudo-counts~\cite{strehl2008analysis,bellemare2016cts,machado2020count}. 
However, such methods only add an intrinsic reward signal to improve sample efficiency in a single task setting.
They do not explicitly learn a policy that is designed to efficiently explore a task.
In contrast, we present a method for explicitly learning an efficient exploration strategy by maximizing the entropy of a single trajectory's state visitation distribution. 
We believe many subareas of \gls{rl} can benefit from such efficient exploration behaviors.
Some applications include Meta~\gls{rl}~\cite{finn2017model, zintgraf2019varibad, rakelly2019efficient, liu2021decoupling}, Continual~\gls{rl}~\cite{khetarpal2020towards,lehnert2017advantages}, and Unsupervised \gls{rl}~\cite{laskin2021urlb}. 
For example, in Meta~\gls{rl} an agent needs to first explore to identify which of the previously observed tasks it is in before the agent can start exploiting rewards in the current task.
In this context, VariBAD~\cite{zintgraf2019varibad} maintains a belief over which task the agent is in given the observed interactions.
While \citeauthor{zintgraf2019varibad} argues that a Bayes-optimal policy implements an efficient exploration strategy, we propose a method that explicitly learns an efficient exploration policy, resulting in discovering rewarding states more efficiently than VariBAD (Section~\ref{sec:experiments}).

The core idea behind \alg{} is the use of the predecessor and successor representations to predict the state visitation distribution induced by a non-Markovian policy for a single finite-length trajectory.
Instead of using the successor representation for transfer, lifelong learning, or learning one representation that solve a set of tasks~\citep{barreto2017successor,zhang2017deep,barreto2018transfer,borsa2018universal,ma2018universal,siriwardhana2019vusfa,hansen2019fast,barreto2020fast,lehnert2020successor,lehnert2020reward,abdolshah2021new,touati2021learning}, we use the successor representation to estimate the state visitation distribution and maximize its entropy.
By using the successor representation in this way, the \alg{} does not rely on density models~\cite{pmlr-v97-hazan19a,lee2019efficient}, an explicit transition model~\cite{tarbouriech2019active, mutti2020intrinsically}, or non-parametric estimators such as k-NN~\cite{mutti2021task}.
In the following sections we will discuss how \alg{} learns a deterministic exploration policy and does not rely on randomization techniques~\citep{mutti2021task,lee2019efficient} or mixing multiple policies to manipulate the state visitation distribution~\cite{lee2019efficient,pmlr-v97-hazan19a}.
Moreover, \citet{mutti2022unsupervised} provide a theoretical analysis proving that efficient (zero regret) exploration is possible with a deterministic non-Markovian policy but computing such a policy is NP-hard.
In this context, \alg{} is to our knowledge the first algorithm for computing such an efficient exploration policy.

\section{Maximum state entropy exploration}
\label{sec:background}

We formalize the exploration task as a \gls{cmp}, a quadruple $\mathcal{M}= \langle \mathcal{S}, \mathcal{A}, p, \mu \rangle$ consisting of a (finite) state space $\mathcal{S}$, a (finite) action space $\mathcal{A}$, a transition function $p$ specifying transition probabilities with $p(s,a,s') = \mathbb{P}( s' | s,a )$, and a start state distribution $\mu$ specifying probability of starting a trajectory at state $s$ with $\mu(s)$.
A trajectory is a sequence $\tau = (s_1,a_1,...,a_{h-1},s_h)$ of some length $h$ that can be simulated in a \gls{cmp}.
A policy $\pi$ specifies the probabilities with which an agent selects actions when simulating a trajectory in a \gls{cmp}.
Typically, this policy is conditioned on the task state and specifies the probabilities of selecting a particular action in standard \gls{rl} algorithms like Q-learning~\cite{watkins1992qlearning,sutton2018reinforcement}. 
However, as illustrated in Figure~\ref{fig:motivation}, the past trajectory (shown with red arrows) determines which next action leads to the best exploratory trajectory.
Consequently, we consider policies that are functions of trajectories rather than just states.

The state visitation frequencies of a trajectory $\tau_h=(s_1,a_1,...,a_{h-1},s_h)$ of length $h$ can be formally expressed in a probability vector by first encoding every state $s_t$ as a one-hot bit vector $\pmb{e}_{s_t}$ and then computing the marginal across time steps for a single trajectory $\tau$:    
\begin{equation}
    \pmb{\xi}_{\gamma,\tau} = \sum_{t=1}^h \gamma(t) \pmb{e}_{s_t}. \label{eq:trajectory-visitation-frequency}
\end{equation}
The marginal in \autoref{eq:trajectory-visitation-frequency} uses a \emph{discount function} $\gamma: \mathbb{N} \to [0, 1]$ (where we denote the set of positive integers with $\mathbb{N}$), such that $\sum_{t=1}^h \gamma(t)=1$. 
We note that this use of a discount function is distinct from using a discount factor in common \gls{rl} algorithms such as Q-learning~\cite{watkins1992qlearning} but using a discount function is necessary as we will elaborate in the following section.

In the example in Figure~\ref{fig:motivation}, the optimal exploratory agent would keep a similar visitation frequency for each state, as illustrated in Figure~\ref{fig:motivation_traj1} where the optimal trajectory traverses every state once within the first 15 steps.
For this trajectory the vector $\pmb{\xi}_{\tau,\gamma}$ would encode a uniform probability vector, given $\gamma(t)=\frac{1}{h}$ for any $t$.
In fact, an optimal exploration policy $\pi^*$ maximizes the entropy of marginal of this probability vector and solves the optimization problem
\begin{equation}
    \pi^* \in \arg \max_\pi H \big( \mathbb{E}_{\tau} \big[ \pmb{\xi}_{\gamma,\tau} \big] \big) \label{eq:entropy-objective-pi}
\end{equation}
where the expectation is computed across trajectories that are simulated in a \gls{cmp} and follow the policy $\pi$.\footnote{
    Here, we consider the Shannon entropy $H(\pmb{p}) = -\sum_i \pmb{p}_i \log \pmb{p}_i$, where the summation ranges over the entries of the probability vector $\pmb{p}$.
}
In the remainder of the paper, we will show how optimizing this objective leads to the uniform sweeping behavior illustrated in \autoref{fig:motivation} and the agent learns to maximize the entropy of the state visitation distribution in a single finite length trajectory.
In the following section, we describe how \alg{} optimizes the objective in \autoref{eq:entropy-objective-pi}.

\section{\alg{}}
\label{sec:algo}

To learn an efficient exploration policy, we need to estimate the state visitation history and predict the distribution over future states.
Consider a trajectory $\tau = (s_1,a_1,...,s_{T-1},a_{T-1},s_{T},...a_{h-1},s_{h})$. 
At an intermediary step $T$, we denote the $T-1$-step prefix with $\tau_{:T-1} = (s_1,a_1,...,s_{T-1})$ and the suffix starting at step $T$ with $\tau_{T:}=(s_T,a_T...,a_{h-1},s_{h})$. 
Using this sub-trajectory notation, the discounted state visitation distribution in \autoref{eq:trajectory-visitation-frequency} can be written as
\begin{align}
    \pmb{\xi}_{\gamma,\tau} &= \sum_{t=1}^{T-1} \gamma(t) \pmb{e}_{s_t} + \sum_{t=T}^h \gamma(t) \pmb{e}_{s_t}.
\end{align}
Assuming the scenario presented in Section~\ref{sec:background}, suppose the agent has followed the trajectory $\tau_{:T}$ until time step $T$.
At this time step, the agent needs to decide which action $a_T$ leads to covering the state space as uniformly as possible and maximizes the entropy of the state visitation distribution. 
The expected state visitation distribution for a policy $\pi$ can be expressed by conditioning on the trace $\tau_{:T}$ and a potential action $a_T \in \mathcal{A}$:
\begin{align}
    &\mathbb{E}_{\tau, \pi} \Big[ \pmb{\xi}_{\gamma,\tau} \Big| \tau_{:T}, a_T \Big] \\
    &=\mathbb{E}_{\tau, \pi} \Bigg[ \sum_{t=1}^{T-1} \gamma(t) \pmb{e}_{s_t} + \sum_{t=T}^h \gamma(t) \pmb{e}_{s_t} \Big| \tau_{:T}, a_T \Bigg] \\
    &= \underbrace{\sum_{t=1}^{T-1} \gamma(t) \pmb{e}_{s_t}}_{=\pmb{\eta}(\tau_{:T-1})} + \underbrace{\mathbb{E}_{\tau_{T+1:}, \pi} \Bigg[  \sum_{t=T}^h\gamma(t) \pmb{e}_{s_t} \Bigg| \tau_{:T},a_T \Bigg]}_{=\pmb{\psi}^\pi( \tau_{:T}, a_T )}\label{eq:objective_sf_pf} ,
\end{align}
where the vector $\pmb{\eta} (\tau_{:T-1})$ is a variant of the predecessor representation~\cite{vanhasselt2021expectedeligibilitytraces,bailey2022predecessor} and the vector $\pmb{\psi}^\pi(\tau_{:T},a_T)$ is a variant of the \glsfirst{sr}~\cite{dayan1993improving}.
Splitting the expected state visitation distribution into a vector $\pmb{\eta}$ and $\pmb{\psi}^\pi$ as outlined in \autoref{eq:objective_sf_pf} is possible because we are assuming a discount function $\gamma$ as defined in Section~\ref{sec:background}.
At time step $T$, the two representations 
can be added together to estimate the expected state visitation probability vector.
Simulating the proposed algorithm is analogous to effectively drawing Monte-Carlo samples from the expectation at different steps $T$ to learn a SR and predict the expected visitation frequencies of $\pmb{\xi}_{\gamma,\tau}$.

The predecessor representation vector $\pmb{\eta}(\tau_{:T-1})$ can still be estimated incrementally similarly to the eligibility trace in the TD($\lambda$) algorithm~\cite{sutton1988learning} (but with a different weighting scheme that uses the discount function $\gamma$).
While the definition of the vector $\pmb{\eta} (\tau_{:T})$ is similar to the definition of eligibility traces~\citep[Chapter 12]{sutton2018reinforcement}, we do not use the predecessor trace for multi-step TD updates to learn more efficiently.
Instead, the vector $\pmb{\eta}(\tau_{:T-1})$ estimates the visitation frequencies of past states---the predecessor states---to decide which states to explore next.

While the predecessor representation can be maintained using an update rule because the observed states are known, predicting future state visitation frequencies is more challenging.  
A potential solution is to exhaustively search through all possible sequences of trajectories starting from the current state.
This is computationally infeasible and requires a dynamics model of the environment.
Moreover, such a model is not always available, and learning them is prone to errors that compound for longer horizons~\cite{ross2011reduction,janner2019trust}.
To this end, we learn a variant of the \glsfirst{sr}, which predicts the expected frequencies of visiting future or successor states under a policy~\cite{dayan1993improving}. 
In contrast to previous methods which learn \gls{sr} conditioned on the current state~\cite{dayan1993improving, barreto2017successor}, \alg{} conditions the \gls{sr} on the entire history of states
\begin{equation}
    \pmb{\psi}^\pi( \tau_{:T}, a_T ) = \mathbb{E}_{\tau_{T+1:}, \pi} \left[  \sum_{t=T}^h\gamma(t) \pmb{e}_{s_t} \middle| \tau_{:T},a_T \right]. \label{eq:successor_representation}
\end{equation}
Conditioning the SR on the trajectory $\tau_{:T}$ is necessary because policy $\pi$ is also conditioned on $\tau_{:T}$ and therefore the visitation frequencies of future states depend on $\tau_{:T}$. 
Moreover, the expectation evaluates all possible trajectories after taking action $a_T$ at time $T$ and following policy $\pi$ afterward. 
We discuss in~\autoref{app:architecture}
how the SR vectors are approximated using a recurrent neural network.

We saw in \autoref{eq:objective_sf_pf} that the predecessor representation and successor representation can be combined to predict the state visitation distribution for a policy $\pi$ and a trajectory-prefix $\tau_{:T}$. 
\alg{} uses the estimated state visitation distribution to compute the entropy term in the objective defined in \autoref{eq:entropy-objective-pi}.
Specifically, the utility function $Q_\text{expl}$ approximates the entropy of the state visitation distribution for an action $a_T$ at every time step.
By defining 
\begin{align}
    \label{eq:Q-function_entropy}
    Q_\text{expl}(\tau_{:T}, a_T) = H\left(\pmb{\eta} (\tau_{:T-1}) + \pmb{\psi}^{\pi} (\tau_{:T}, a_T)\right),
\end{align}
the action that leads to the highest state visitation entropy is assigned the highest utility value.
Notably, the proposed Q-function differs from prior methods using the \gls{sr}, as we neither factorize the reward function~\cite{barreto2017successor,barreto2018transfer,borsa2018universal,touati2021learning} nor use the \gls{sr} for learning a state abstraction~\cite{lehnert2020successor}. 
Optimizing the exploration Q-function stated in~\autoref{eq:Q-function_entropy} is challenging as it depends on the \gls{sr} that itself depends on the policy $\pi$ which changes during learning.
Furthermore,~\citet{guo2021geometric} that the Shannon-entropy based objective is difficult to directly optimize using gradient-based methods (due to the log term inside an expectation)~\citep{lee2019efficient,pong2019skew,guo2021geometric}. 
In contrast, we outline in the following paragraphs how the entropy objective in \autoref{eq:Q-function_entropy} can be directly optimized using either a Q-learning~\citep{watkins1992qlearning} style method or a method based on the Deterministic Policy Gradient~\cite{silver2014deterministic} framework for finite and infinite action spaces.

\paragraph{Finite action space framework}
Since the predecessor representation is fixed for a given trajectory $\tau_{:T}$, optimizing the Q-function defined in~\autoref{eq:Q-function_entropy} boils down to predicting the optimal \gls{sr} for a given history $\tau_{:T}$. 
Similar to prior Successor Feature learning methods~\citep{barreto2017successor,barreto2018transfer,lehnert2017advantages}, we approximate the \gls{sr} with a parameterized and differentiable function $\pmb{\psi}_{\pmb{\theta}}$ and use a loss based on a one-step temporal difference error.
Given an approximation $\pmb{\psi}_{\pmb{\theta}}$, the \gls{sr} prediction target is obtained by the current state embedding and \gls{sr} of the optimal action at the next step:
\begin{equation}
    \pmb{y}(\tau_{:T+1},a'_{T+1}) = \pmb{e}_{s_{
T}} + \gamma(T+1) \pmb{\psi_{\theta}} (\tau_{:T+1}, a'_{T+1}),
\end{equation}
where $\tau_{:T+1}$ is obtained by adding action $a_T$ and the received next state $s_{T+1}$ to the trajectory $\tau_{:T}$. 
Analogous to Q-Learning~\cite{watkins1992qlearning}, the optimal action at the next step is specified by
\begin{align}
    \label{eq:optimal_policy}
    a'_{T+1}=\arg\max_{a \in \mathcal{A}} Q_{\text{expl}} (\tau_{:T+1}, a). 
\end{align}
Being greedy with respect to these entropy values to estimate the target leads to improving the policy $\pi$ which in turn finds the \gls{sr} for the optimal policy.
(Appendix~\ref{app:convergence} presents a convergence analysis of this method in a dynamic programming setting.)
Then, the function $\pmb{\psi}_{\pmb{\theta}}$ is optimized using gradient descent on the loss function $\mathcal{L}_{SR}$, given by
\begin{align}
    \label{eq:loss_sr}
    \mathcal{L}_{SR} &=  || \pmb{\psi}_{\pmb{\theta}} (\tau_{:T}, a_{T})  - \pmb{y}(\tau_{:T+1},a'_{T+1}) ||^2,
\end{align}
where gradients are not propagated through the target $\pmb{y}(\tau_{:T+1},a'_{T+1})$.
Finally, the optimal policy selects actions greedily with respect to the $Q_{\text{expl}}$ function. 
Algorithm~\ref{algo:epl_qlearning} describes the training procedure for the proposed variant for finite action spaces.

\paragraph{Infinite action space framework} 
Directly obtaining gradient estimates of objective defined in \autoref{eq:entropy-objective-pi} is challenging because of the expectation term in the non-linear logarithmic term. 
Previous approaches have used alternative optimization methods~\cite{lee2019efficient,pong2019skew} or resorted to a simpler noise-contrastive objective function~\cite{guo2021geometric}. 
In contrast with prior algorithms, we derived an \alg{} variant for infinite action spaces that optimizes an actor-critic architecture using policy gradients to maximise the maximum state entropy objective.
The agent uses an actor-critic architecture where actor and critic networks are conditioned on the history of visited states.
The actor~$\pi_{\mu}(\tau)$ is parameterized with a parameter vector $\mu$ and is a deterministic map from a trajectory to an action.
The critic predicts the utility function conditioned on a given trajectory and action.
Similar to the finite action space variant, the predecessor representation is fixed for a given trajectory and the network has to predict \gls{sr}~$\pmb{\psi}_{\theta}(\tau, a)$ for a given trajectory~$\tau$ and action~$a$.
Here, the target value of \gls{sr} to update the critic is specified by the action obtained using the policy~$a'_{T+1}=\pi_{\mu}(\tau_{:T+1})$, given by
\begin{align}
    \pmb{y} = \pmb{e}_{s_T} + \gamma(T+1)\pmb{\psi}_{\theta}(\tau_{:T+1}, a'_{T+1}).
\end{align}
The critic is trained with the same loss function $\mathcal{L}_{SR}$ as defined in \autoref{eq:loss_sr}, where the gradients are not propagated through the target.
The actor is optimized to maximize the estimated utility function (\autoref{eq:entropy-objective-pi}).
Since the actor is deterministic, policy gradients are computed using an adaptation of the deterministic policy gradient theorem~\cite{silver2014deterministic}.
Because the actor network has no dependency on predecessor trace (which depends on the observed states only), gradients for the actor parameters are obtained by applying chain rule leading to the following gradient of~\autoref{eq:Q-function_entropy} (please refer to Proposition \autoref{prop:pg} for more details on the derivation):
\begin{align}
    \label{eq:q_dpg}
    \nabla_{\mu} J(\pi_{\mu}) = \mathbb{E}_{\tau\sim\rho} \Big[ \sum_i z_i \nabla_{\mu} \pi_{\mu} (\tau) \nabla_a \pmb{\psi}_i(\tau,a) \big|_{a=\pi_{\mu}(\tau)}\Big],
\end{align}
where $z_i=-\log[\pmb{\eta}(\tau_{:-1})_i + \pmb{\psi}(\tau, \pi_{\mu}(\tau))_i] - 1$ is the multiplicative factor for state $i$, and depends on the expected probability of visiting a state.
The factor $z_i$ can take values between between -1 and $\infty$, with positive values of high magnitude for states with low visitation probability and negative values for state with high probability of visitation.
Thus, the factor $z_i$ scales the policy gradients to maximize the entropy of the state visitation distribution.
In Algorithm~\autoref{algo:epl_pg} we outline the training procedure for the infinite action space framework.
 
\section{Experiments}
\begin{figure*}[!t]
    \centering
    \includegraphics[width=0.99\textwidth]{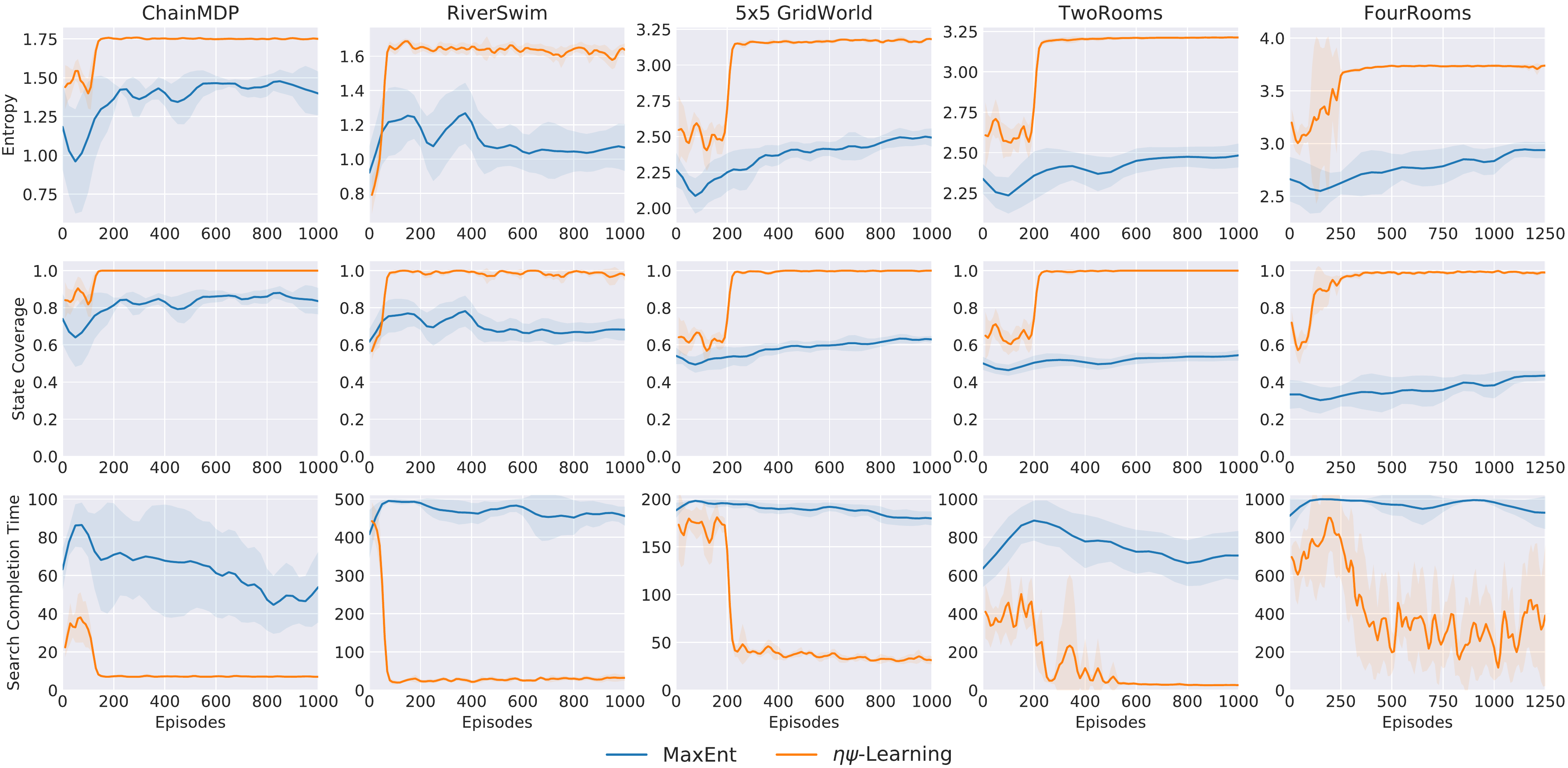}
    \caption{
        Comparison of \alg{} and MaxEnt~\cite{pmlr-v97-hazan19a} on three metrics: 
        \emph{Entropy} (top row) of state visitation distribution, \emph{State Coverage} (middle row) representing the fraction of state space visited, and \emph{Search Completion Time} (bottom row) denoting steps taken to cover the state space.
    } 
    \label{fig:results}
\end{figure*}
\label{sec:experiments}
To analyze and test if \alg{} learns an efficient exploration policy, we evaluate the proposed method on a set of discrete and continuous control tasks.
In these experiments, we are recording a set of different performance measures to access if the resulting exploration policies do in fact maximize the entropy of the state visitation distribution and if most states are explored by \alg{}. 
The following results demonstrate that by maintaining a predecessor representation and conditioning the \gls{sr} on the simulated trajectory prefix, the \alg{} agent learns a deterministic exploration policy that minimizes the number of interactions needed to visit all states.
In addition, we expand our method to continuous control tasks and demonstrate how \alg{} can efficiently explore in complex domains with infinite action space.
Our method is ideal for searching out rewards in difficult sparse reward environments. 
We compare \alg{}, which learns to explore an environment as efficiently as possible, to recent meta-\gls{rl} methods~\cite{zintgraf2019varibad} that aim to learn how to optimally explore an environment to infer the rewarding or goal state.

\textbf{Environments}: 
We experiment with different tasks with both finite and infinite action spaces.
The \textbf{ChainMDP} and \textbf{RiverSwim}~\cite{strehl2008analysis} is a six-state chain where the transitions are deterministic or stochastic, respectively.
In these tasks a Markovian policy cannot cover the state space uniformly because the agent has to pace back and forth along the chain, visiting the same state multiple times.
For the RiverSwim environment , a non-stationary policy, a policy that is a function of the time step, cannot optimally cover all states because non-determinism in the transitions can place the agent into different states at random.
Furthermore, we include the \textbf{$5 \times 5$ grid world} example used in \autoref{fig:motivation}.
We also test \alg{} on two harder exploration tasks---the \textbf{TwoRooms} and \textbf{FourRooms} domains, which are challenging because it is not possible to obtain exact uniform visitation distribution due to the wall structure.
For continuous control tasks, we evaluate on Reacher and Pusher tasks, where the agent has a robotic-arm with multiple joints.
The task is to maximize the entropy over the locations covered by the fingertip of the robotic-arm. 
\autoref{app:environments} provides more details on the environments and the hyper-parameters are reported in \autoref{app:hyperparameter}.

\textbf{Prior Methods}: 
To our knowledge, existing work focusses on learning Markovian exploration policies~\citep{mutti2022maximumentropyexploration}.
We use MaxEnt~\cite{pmlr-v97-hazan19a} as a baseline agent for our experiments because this method optimizes a similar maximum entropy objective as \alg{}---with the difference that MaxEnt learns a Markovian policy and resorts to randomization to obtain a close to uniform state visitation distribution.
A comparison with SMM~\citep{lee2019efficient} and MEPOL~\citep{mutti2021task} is skipped because these methods optimize a similar maximum entropy objective with a Markovian policy and cannot express the same exploration behaviour as \alg{}.
\begin{figure*}[t]
\centering
\subfigure[]{
\includegraphics[width=.3\textwidth]{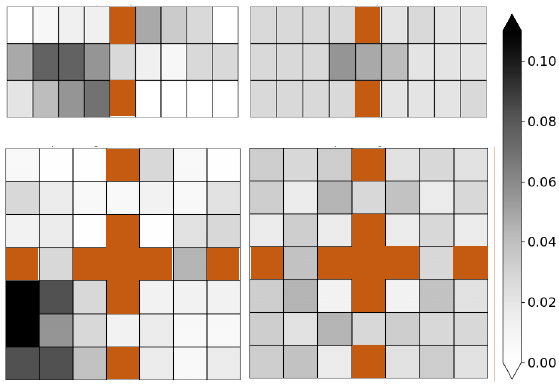}
\label{fig:vis_heatmap_2envs}
}
\subfigure[]{
\includegraphics[width=.66\textwidth]{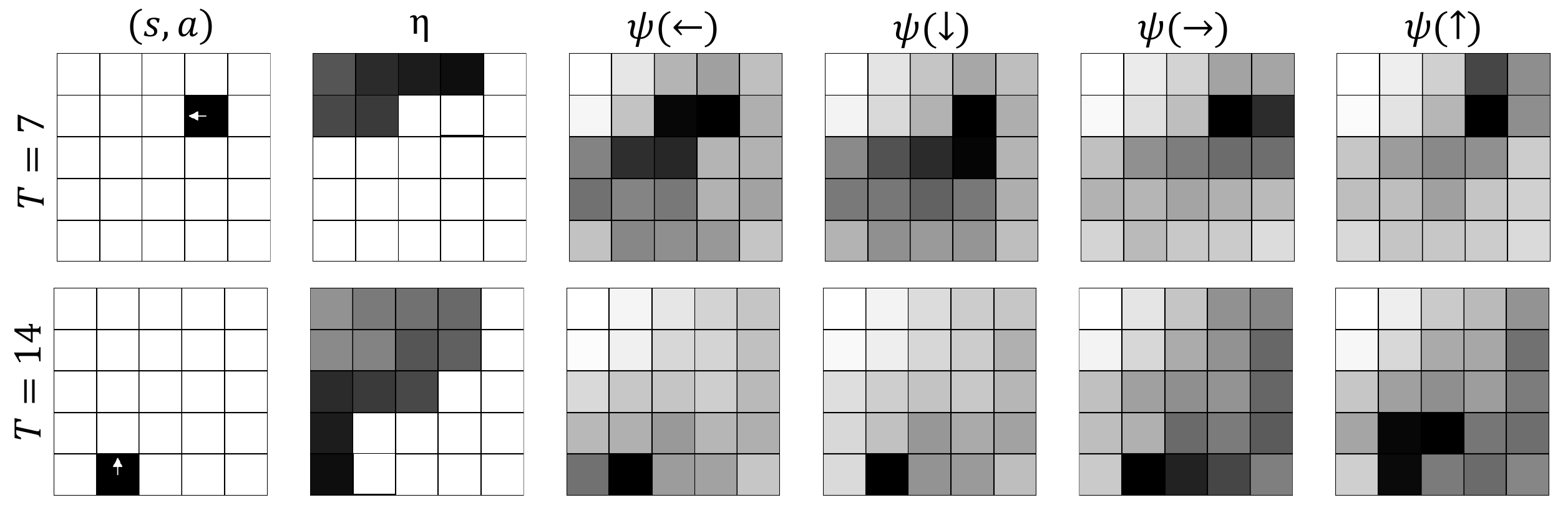}
\label{fig:sf_map}
}

\caption{(a) Heatmap of state visitation distribution by unrolling a trajectory using MaxEnt (left) and \alg{} (right) on TwoRooms and FourRooms environments. (b) Visualization of learned \gls{sr} of each action (denoted with $\psi(.)$) at time steps $T=7,14$ for a trajectory using \alg{} on $5\times5$ grid. $(s, a)$ denotes the state (black) and action taken by the agent (direction of white arrow), $\pmb{\eta}$ is the predecessor representation till time $T$ (higher values have darker shade)}.
\label{fig:ablations}
\end{figure*}

\textbf{Evaluation Metrics}: 
\textit{Entropy} measures a method's ability to have similar visitation frequency for each state in the state space.
This signifies the gap between the observed state visitation distribution and the optimal distribution that maximizes the entropy term.
The Entropy metric is computed using the objective defined in \autoref{eq:entropy-objective-pi} over a single trajectory generated by the agent. 
A constant discount factor of $\gamma(t) = \frac{1}{h}$ is used to obtain the state visitation distribution during evaluation.
An agent can maximize this measure without actually exploring all states of an environment---a desirable property for \gls{rl} where rewards may be sparse and hidden in complex to-reach states.
We record the \textit{state coverage} metric which represents the fraction of states in the environment visited by the agent at least once within a trajectory.
Lastly, we want agents to explore the state space efficiently.
For example, an optimal agent can sweep through the gridworld presented in \autoref{fig:motivation} with a search completion time of $15$ steps~( Figure~\ref{fig:motivation_traj1} shows an optimal trajectory).
The \emph{search completion time} metric measures the steps taken to discover each state in the environment.
All results report the mean performance computed over 5 random seeds with 95\% confidence intervals shading.

\textbf{Quantitative Results}:
\autoref{fig:results} presents the results obtained for \alg{} and MaxEnt~\cite{pmlr-v97-hazan19a}. 
Compared to MaxEnt, which learns a Markovian policy, \alg{} achieves 20-50\% higher entropy.
This indicates that the MaxEnt algorithm by learning a Markovian and stochastic policy was randomizing at certain states which lead to sub-optimal behaviors.
The performance gain was more prominent in grid-based environments because the MaxEnt agent was visiting some states more frequently than others, which are harder to explore efficiently.
Furthermore, high \textit{entropy} values suggest that the agent visits states with similar frequency in the environment and does not get stuck at a particular state.
We attribute this behavior of \alg{} to the proposed Q-function that picks action to visit different states and maximize the objective.
\begin{figure*}[t]
\centering
\subfigure[]{
\includegraphics[width=.37\textwidth]{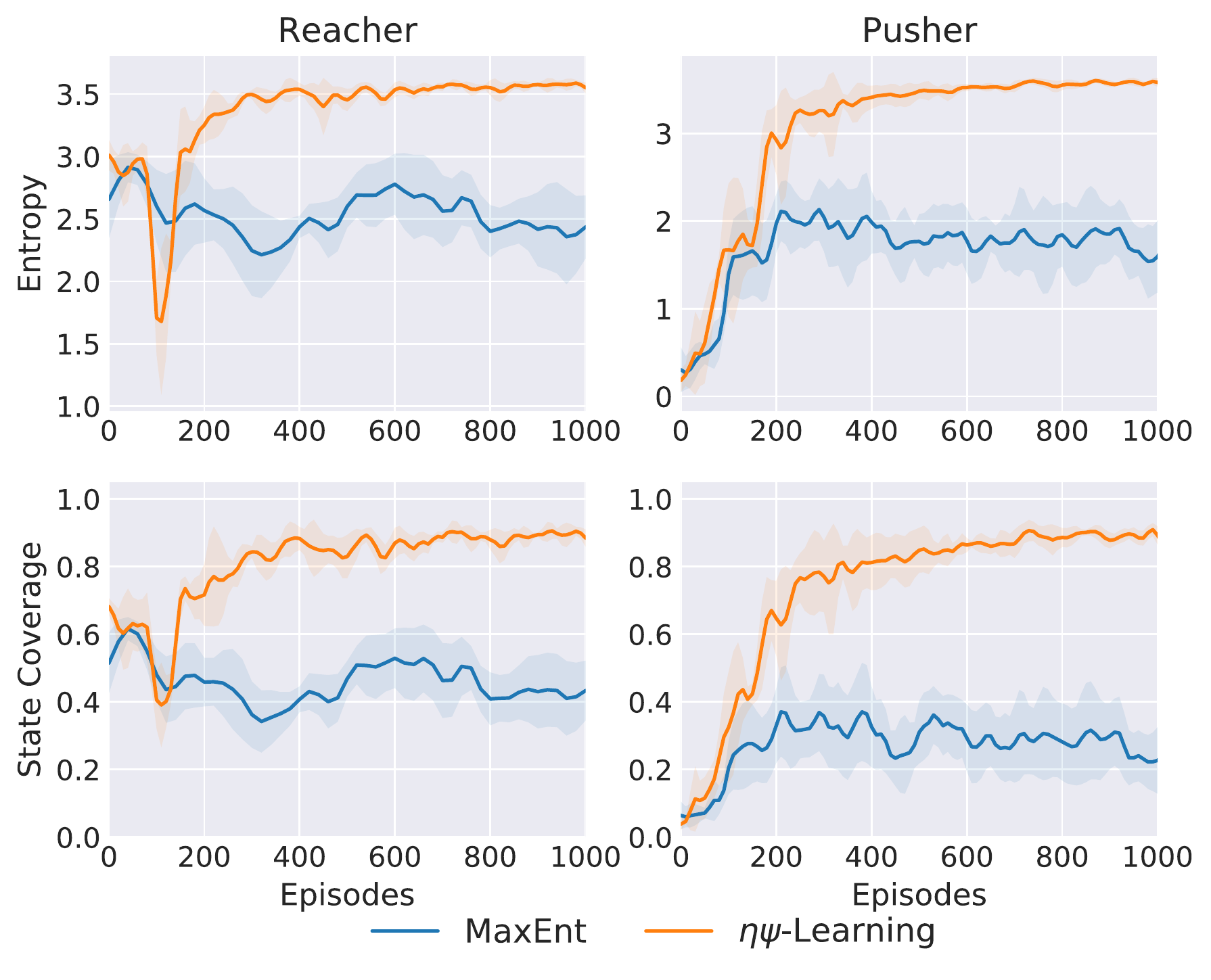}
\label{fig:infnty_action_results}
}
\subfigure[]{
\includegraphics[width=.35\textwidth]{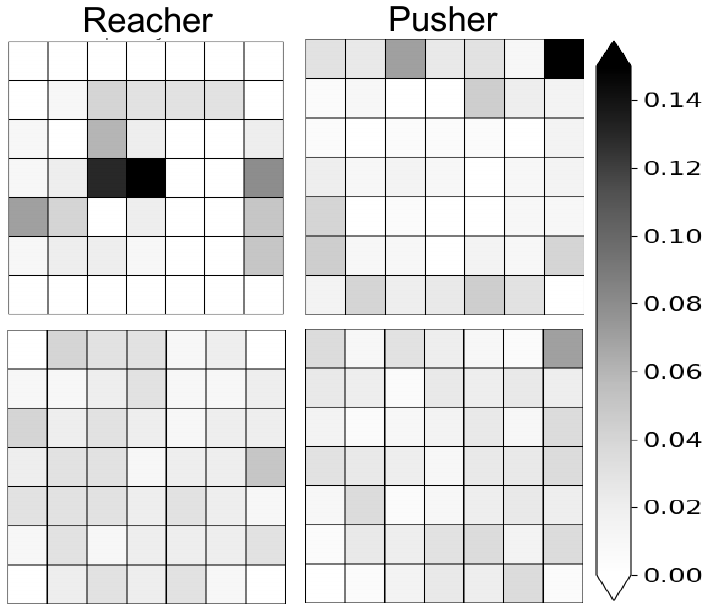}
\label{fig:infnty_action_heatmap}
}
\subfigure[]{
\includegraphics[width=.20\textwidth]{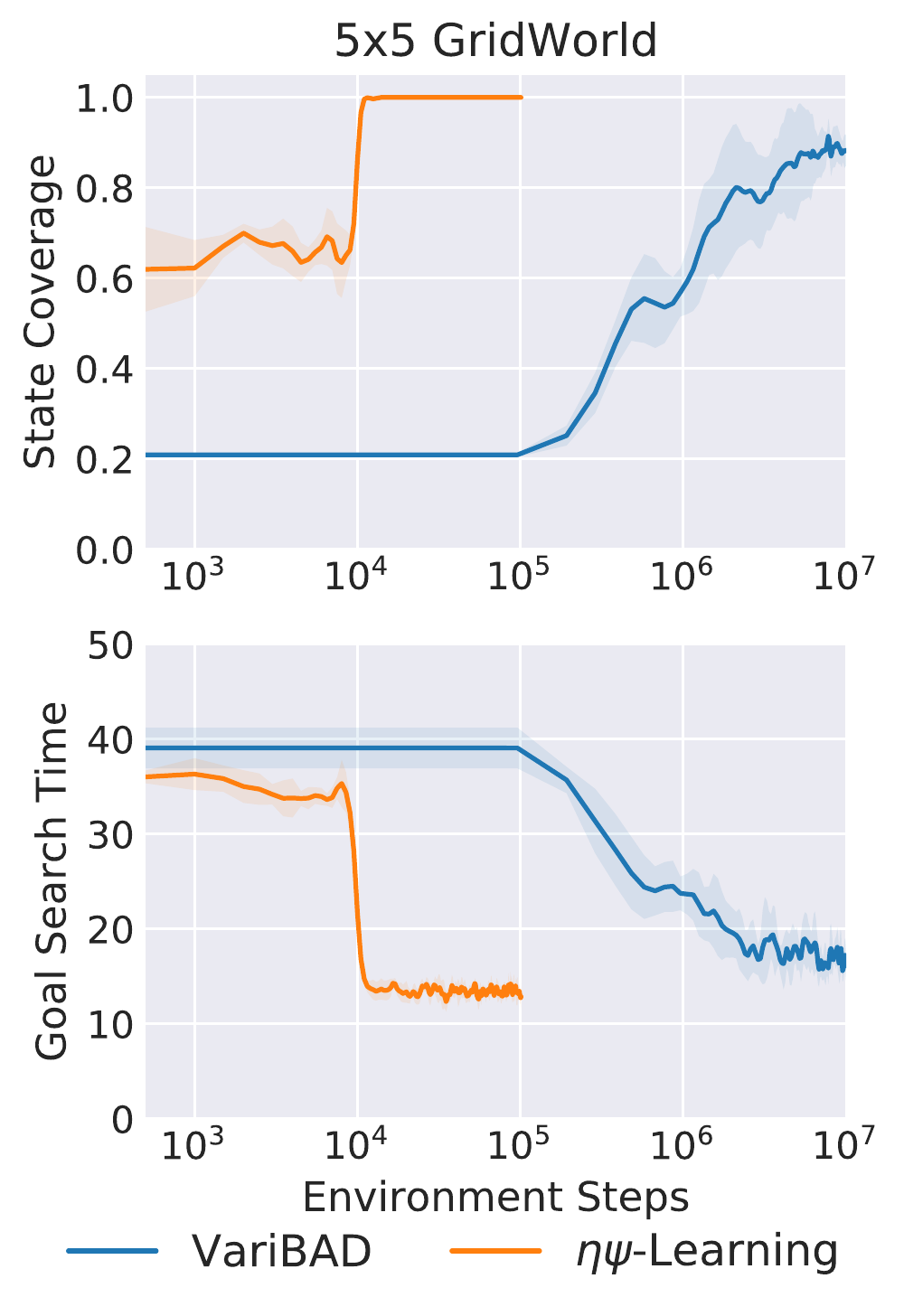}
\label{fig:metarl}
}
\caption{(a) Comparison of \alg{} and MaxEnt~\cite{mutti2022maximumentropyexploration} on Reacher and Pusher environments (b) Heatmaps of state visitation distribution of MaxEnt (top) and \alg{} (bottom), (c) comparison with VariBAD~\cite{zintgraf2019varibad} on State Coverage and Goal Search Time metrics.
}
\label{fig:cont_control_metarl}
\end{figure*}

\autoref{fig:results} also shows that \alg{} achieves optimal state coverage across environments exemplifying that \alg{} while maximizing the entropic measure also learns to cover the state space within a single trajectory. 
However, the baseline MaxEnt was not able to discover all the states in the environment.
MaxEnt was unable to visit all the states in ChainMDP and RiverSwim environments with trajectory length of $20$ and $50$, respectively.
Moreover, the state coverage of MaxEnt was around \textbf{50-60\%} on the harder TwoRooms and FourRooms tasks, where the agent has to navigate between different rooms and is required to remember the order of visiting different rooms.
These results reveal that Markovian policy limits an agent's ability to maximize the state coverage in a task. 
The proposed method also outperformed the baseline on the \textit{search completion time} metric across all environments.
Notably, on ChainMDP, the $5 \times 5$ gridworld, and TwoRooms environments, \alg{} converged within 500 episodes. 
However, \alg{} did not achieve optimal \textit{search exploration time} on the FourRooms environment as it missed a spot in a room and resorted to it later in the episode.

To further understand the gains offered by \alg{}, we visualized the state visitation distributions on a single trajectory (in Figure~\ref{fig:vis_heatmap_2envs}).
On the TwoRooms environment, \alg{} had similar density on the states in both the rooms, where the density is more around the center.
This is because the agent was sweeping across rooms alternatively. 
\alg{} showed a better-visited state distribution on the FourRooms environment with more distributed density across states.
However, MaxEnt was not visiting all states and also visited a few states more frequently than others, elucidating the lower performance on entropy and state coverage.
We further visualized the learned \gls{sr} to see if \alg{} learns a \gls{sr} for the optimal exploration policy through generalized policy improvement~\cite{barreto2017successor}.
For this analysis, we sampled a trajectory on $5\times5$ gridworld.
Figure~\ref{fig:sf_map} reports the heatmaps of the learned \gls{sr} vector for each action at different steps in the trajectory.
We observe that the \gls{sr} vector for each action has lower density on the states already observed in the trace.
This exemplifies that the learned \gls{sr} captures the history of the visited states that further aids in taking actions to maximize the entropy of state visitation distribution. 
We also study if MaxEnt can show similar gains when trained with a recurrent policy (Appendix~\ref{app:maxent_recurrent}) and compared the agents when evaluated across multiple trajectories~(Appendix~\ref{app:multiple_trajs}). 

\textbf{Continuous Control tasks}: 
The efficacy of \alg{} is further tested on environments with infinite action space. Figure~\ref{fig:infnty_action_results} reports the Entropy and State Coverage metric on Reacher and Pusher environments, where \alg{} outperformed the baseline MaxEnt on both metrics.
The gains are more significant on the Pusher environment which is a harder task because of multiple hinges in the robotic-arm.
The proposed method~\alg{} achieves close to \textbf{90\%} coverage in both environments, whereas the MaxEnt had only close to \textbf{50\%} and \textbf{40\%} coverage on Reacher and Pusher environments, respectively.
In Figure~\ref{fig:infnty_action_heatmap}, heatmaps of the state visitation density for a single trajectory shows that \alg{}~has more uniformly distributed density compared to MaxEnt.
The Pusher environment has high density at the top-right corner of the grid denoting the time taken by the agent to move the fingertip to other locations from the starting state. 
Notably, the proposed method \alg{}~has lower density at the starting state and we believe that conditioning on the history of states is guiding the agent to move the robotic-arm to other locations to maximize the entropy over the state space.
In Appendix~\ref{app:reacher_pusher_trajs}, a visualization of a rolled-out trajectory generated using \alg{} is presented showing that the agent learns to efficiently maneuver the fingertip of the robotic-arm to different locations in the environment.

\textbf{Comparison with Meta-\gls{rl}}:
A question central to Meta-\gls{rl}~\cite{finn2017model,zintgraf2019varibad,liu2021decoupling} is the ability to quickly explore a task and find rewarding states in complex tasks where the rewards are sparse.
In this context,~\citet{zintgraf2019varibad} present the VariBAD method, which maintains a belief over different tasks to infer the optimal policy---leading to efficient exploration behaviour that enables the agent to discover rewarding states quickly.
Similar to VariBAD, the predecessor representation $\pmb{\eta}$ in \alg{} keeps track of which states have been explored and which states are not explored.
Figure~\ref{fig:metarl} compares the exploration behaviour of \alg{} to VariBAD:
In terms of the State Coverage and Goal Search Time metric, \alg{} outperforms VariBAD significantly because \alg{} is designed to optimize the entropy of the state visitation frequencies of a single trajectory instead of performing Bayes-adaptive inference across a task space.
We refer the reader to \autoref{app:metarl} for more details. 

\section{Discussion}
To explore efficiently, an intelligent agent needs to consider past episodic experiences to decide on the next exploratory action.
We demonstrate how the predecessor representation---an encoding of past state visitations---can be combined with the successor representation---a prediction of future state visitations---to learn efficient exploration policies that maximize the state-visitation-distribution entropy.
Across a set of different environments, we illustrate how \alg{} consistently reasons across different trajectories to explore near optimally---a task that is NP-hard~\cite{mutti2022maximumentropyexploration}. 

To the best of our knowledge, \alg{}~is the first algorithm that combines predecessor and successor representations to estimate the state visitation distribution. 
Furthermore, \alg{} learns a non-Markovian policy and can therefore express exploration behavior not afforded by existing methods~\cite{pmlr-v97-hazan19a,lee2019efficient,mutti2021task,guo2021geometric}.
To further increase the applicability of \alg{}, one interesting direction of future research is to extend \alg{} to POMDP environments where states are either partially observable or complex such as images. This is challenging because the agent has to learn an embedding of state observations that capture only the relevant components of the state space to maximize the entropy. 
We believe a promising approach would be to leverage the idea of \glsfirst{sm}~\cite{touati2021learning,touati2022does,farebrother2023proto} which have shown promising results when scaled to high-dimensional inputs like images. 
Furthermore, the presented approach can be also used for designing other algorithms that control the state visitation distribution. 
An application is goal-conditioned RL, where the agents need to minimize the KL divergence between visitation distribution of policy and goal-distribution~\cite {lee2019efficient,pong2019skew}. 
Another application is Safe RL~\citep{wagener2021safe} where agents receive a penalty upon visiting unsafe states to avoid observing them. 

We study reinforcement learning, which aims to enable autonomous agents to acquire search behaviors. 
This study of developing exploration behaviors in reinforcement learning is guided by a fundamental curiosity about the nature of autonomous learning; it has a number of potential practical applications and broad implications. 
First, autonomous exploration for pre-training in general, can enable autonomous agents to acquire useful skills with less human intervention and effort, potentially improving the feasibility of learning-enabled robotic systems. 
Second, the practical applications that we illustrate, such as applications to continuous environments, can accelerate reinforcement learning in certain settings. 
Specific to our method, finite length entropy maximization may also in the future offer a useful tool for search and rescue, by equipping agents with an objective that causes them to explore a space systematically to locate lost items. 
However, these types of reinforcement learning methods also have a number of uncertain broad implications: agents that explore the environment and attempt to acquire open-ended skills may carry out unexpected or unwanted behaviors, and would require suitable safety mechanisms of their own during training.

\section*{Acknowledgements}
The authors would like to thank Harley Wiltzer for his valuable feedback and discussions. The writing of the paper also benefited from discussions with Darshan Patil, Chen Sun, Mandana Samiei, Vineet Jain, and Arushi Jain. AJ and IR acknowledge the support from Canada CIFAR AI Chair Program and from the Canada Excellence Research Chairs (CERC) program. The authors are also grateful to Mila (mila.quebec) and Digital Research Alliance of Canada for computing resources.

\bibliographystyle{icml2023}

\begin{thebibliography}{64}
\providecommand{\natexlab}[1]{#1}
\providecommand{\url}[1]{\texttt{#1}}
\expandafter\ifx\csname urlstyle\endcsname\relax
  \providecommand{\doi}[1]{doi: #1}\else
  \providecommand{\doi}{doi: \begingroup \urlstyle{rm}\Url}\fi

\bibitem[Abdolshah et~al.(2021)Abdolshah, Le, George, Gupta, Rana, and
  Venkatesh]{abdolshah2021new}
Abdolshah, M., Le, H., George, T.~K., Gupta, S., Rana, S., and Venkatesh, S.
\newblock A new representation of successor features for transfer across
  dissimilar environments.
\newblock In \emph{International Conference on Machine Learning}, pp.\  1--9.
  PMLR, 2021.

\bibitem[Bailey \& Mattar(2022)Bailey and Mattar]{bailey2022predecessor}
Bailey, D. and Mattar, M.
\newblock Predecessor features.
\newblock \emph{arXiv preprint arXiv:2206.00303}, 2022.

\bibitem[Barreto et~al.(2017)Barreto, Dabney, Munos, Hunt, Schaul, van Hasselt,
  and Silver]{barreto2017successor}
Barreto, A., Dabney, W., Munos, R., Hunt, J.~J., Schaul, T., van Hasselt,
  H.~P., and Silver, D.
\newblock Successor features for transfer in reinforcement learning.
\newblock \emph{Advances in neural information processing systems}, 30, 2017.

\bibitem[Barreto et~al.(2018)Barreto, Borsa, Quan, Schaul, Silver, Hessel,
  Mankowitz, Zidek, and Munos]{barreto2018transfer}
Barreto, A., Borsa, D., Quan, J., Schaul, T., Silver, D., Hessel, M.,
  Mankowitz, D., Zidek, A., and Munos, R.
\newblock Transfer in deep reinforcement learning using successor features and
  generalised policy improvement.
\newblock In \emph{International Conference on Machine Learning}, pp.\
  501--510. PMLR, 2018.

\bibitem[Barreto et~al.(2020)Barreto, Hou, Borsa, Silver, and
  Precup]{barreto2020fast}
Barreto, A., Hou, S., Borsa, D., Silver, D., and Precup, D.
\newblock Fast reinforcement learning with generalized policy updates.
\newblock \emph{Proceedings of the National Academy of Sciences}, 117\penalty0
  (48):\penalty0 30079--30087, 2020.

\bibitem[Bellemare et~al.(2016)Bellemare, Srinivasan, Ostrovski, Schaul,
  Saxton, and Munos]{bellemare2016cts}
Bellemare, M., Srinivasan, S., Ostrovski, G., Schaul, T., Saxton, D., and
  Munos, R.
\newblock Unifying count-based exploration and intrinsic motivation.
\newblock In \emph{Advances in Neural Information Processing Systems}, pp.\
  1471--1479, 2016.

\bibitem[Bengio et~al.(1994)Bengio, Simard, and Frasconi]{bengio1994learning}
Bengio, Y., Simard, P., and Frasconi, P.
\newblock Learning long-term dependencies with gradient descent is difficult.
\newblock \emph{IEEE transactions on neural networks}, 5\penalty0 (2):\penalty0
  157--166, 1994.

\bibitem[Borsa et~al.(2018)Borsa, Barreto, Quan, Mankowitz, Munos, Van~Hasselt,
  Silver, and Schaul]{borsa2018universal}
Borsa, D., Barreto, A., Quan, J., Mankowitz, D., Munos, R., Van~Hasselt, H.,
  Silver, D., and Schaul, T.
\newblock Universal successor features approximators.
\newblock \emph{arXiv preprint arXiv:1812.07626}, 2018.

\bibitem[Burda et~al.(2019)Burda, Edwards, Storkey, and
  Klimov]{burda2019exploration}
Burda, Y., Edwards, H., Storkey, A., and Klimov, O.
\newblock Exploration by random network distillation.
\newblock In \emph{International Conference on Learning Representations}, 2019.
\newblock URL \url{https://openreview.net/forum?id=H1lJJnR5Ym}.

\bibitem[Cho et~al.(2014)Cho, Van~Merri{\"e}nboer, Gulcehre, Bahdanau,
  Bougares, Schwenk, and Bengio]{cho2014gru}
Cho, K., Van~Merri{\"e}nboer, B., Gulcehre, C., Bahdanau, D., Bougares, F.,
  Schwenk, H., and Bengio, Y.
\newblock Learning phrase representations using rnn encoder-decoder for
  statistical machine translation.
\newblock \emph{arXiv preprint arXiv:1406.1078}, 2014.

\bibitem[Christopher(1992)]{watkins1992qlearning}
Christopher, J.
\newblock Watkins and peter dayan.
\newblock \emph{Q-Learning. Machine Learning}, 8\penalty0 (3):\penalty0
  279--292, 1992.

\bibitem[Clark(2018)]{clark2018nice}
Clark, A.
\newblock A nice surprise? predictive processing and the active pursuit of
  novelty.
\newblock \emph{Phenomenology and the Cognitive Sciences}, 17\penalty0
  (3):\penalty0 521--534, 2018.

\bibitem[Dayan(1993)]{dayan1993improving}
Dayan, P.
\newblock Improving generalization for temporal difference learning: The
  successor representation.
\newblock \emph{Neural computation}, 5\penalty0 (4):\penalty0 613--624, 1993.

\bibitem[Ecoffet et~al.(2019)Ecoffet, Huizinga, Lehman, Stanley, and
  Clune]{ecoffet2019go}
Ecoffet, A., Huizinga, J., Lehman, J., Stanley, K.~O., and Clune, J.
\newblock Go-explore: a new approach for hard-exploration problems.
\newblock \emph{arXiv preprint arXiv:1901.10995}, 2019.

\bibitem[Farebrother et~al.(2023)Farebrother, Greaves, Agarwal, Lan, Goroshin,
  Castro, and Bellemare]{farebrother2023proto}
Farebrother, J., Greaves, J., Agarwal, R., Lan, C.~L., Goroshin, R., Castro,
  P.~S., and Bellemare, M.~G.
\newblock Proto-value networks: Scaling representation learning with auxiliary
  tasks.
\newblock \emph{arXiv preprint arXiv:2304.12567}, 2023.

\bibitem[Finn et~al.(2017)Finn, Abbeel, and Levine]{finn2017model}
Finn, C., Abbeel, P., and Levine, S.
\newblock Model-agnostic meta-learning for fast adaptation of deep networks.
\newblock In \emph{International conference on machine learning}, pp.\
  1126--1135. PMLR, 2017.

\bibitem[Fujimoto et~al.(2018)Fujimoto, van Hoof, and Meger]{fujimoto2018td3}
Fujimoto, S., van Hoof, H., and Meger, D.
\newblock Addressing function approximation error in actor-critic methods.
\newblock \emph{arXiv preprint arXiv:1802.09477}, 2018.

\bibitem[Guo et~al.(2021)Guo, Azar, Saade, Thakoor, Piot, Pires, Valko,
  Mesnard, Lattimore, and Munos]{guo2021geometric}
Guo, Z.~D., Azar, M.~G., Saade, A., Thakoor, S., Piot, B., Pires, B.~A., Valko,
  M., Mesnard, T., Lattimore, T., and Munos, R.
\newblock Geometric entropic exploration.
\newblock \emph{arXiv preprint arXiv:2101.02055}, 2021.

\bibitem[Haarnoja et~al.(2018)Haarnoja, Zhou, Abbeel, and
  Levine]{haarnoja2018soft}
Haarnoja, T., Zhou, A., Abbeel, P., and Levine, S.
\newblock Soft actor-critic: Off-policy maximum entropy deep reinforcement
  learning with a stochastic actor.
\newblock In \emph{International conference on machine learning}, pp.\
  1861--1870. PMLR, 2018.

\bibitem[Hafner et~al.(2020)Hafner, Lillicrap, Norouzi, and
  Ba]{hafner2020dreamerv2}
Hafner, D., Lillicrap, T., Norouzi, M., and Ba, J.
\newblock Mastering atari with discrete world models.
\newblock \emph{arXiv preprint arXiv:2010.02193}, 2020.

\bibitem[Hansen et~al.(2019)Hansen, Dabney, Barreto, Van~de Wiele,
  Warde-Farley, and Mnih]{hansen2019fast}
Hansen, S., Dabney, W., Barreto, A., Van~de Wiele, T., Warde-Farley, D., and
  Mnih, V.
\newblock Fast task inference with variational intrinsic successor features.
\newblock \emph{arXiv preprint arXiv:1906.05030}, 2019.

\bibitem[Hazan et~al.(2019)Hazan, Kakade, Singh, and
  Van~Soest]{pmlr-v97-hazan19a}
Hazan, E., Kakade, S., Singh, K., and Van~Soest, A.
\newblock Provably efficient maximum entropy exploration.
\newblock In Chaudhuri, K. and Salakhutdinov, R. (eds.), \emph{Proceedings of
  the 36th International Conference on Machine Learning}, volume~97 of
  \emph{Proceedings of Machine Learning Research}, pp.\  2681--2691. PMLR,
  09--15 Jun 2019.
\newblock URL \url{https://proceedings.mlr.press/v97/hazan19a.html}.

\bibitem[Janner et~al.(2019)Janner, Fu, Zhang, and Levine]{janner2019trust}
Janner, M., Fu, J., Zhang, M., and Levine, S.
\newblock When to trust your model: Model-based policy optimization.
\newblock \emph{Advances in Neural Information Processing Systems}, 32, 2019.

\bibitem[Khetarpal et~al.(2020)Khetarpal, Riemer, Rish, and
  Precup]{khetarpal2020towards}
Khetarpal, K., Riemer, M., Rish, I., and Precup, D.
\newblock Towards continual reinforcement learning: A review and perspectives.
\newblock \emph{arXiv preprint arXiv:2012.13490}, 2020.

\bibitem[Kidd \& Hayden(2015)Kidd and Hayden]{kidd2015psychology}
Kidd, C. and Hayden, B.~Y.
\newblock The psychology and neuroscience of curiosity.
\newblock \emph{Neuron}, 88\penalty0 (3):\penalty0 449--460, 2015.

\bibitem[Kingma \& Ba(2014)Kingma and Ba]{kingma2014adam}
Kingma, D.~P. and Ba, J.
\newblock Adam: A method for stochastic optimization.
\newblock \emph{arXiv preprint arXiv:1412.6980}, 2014.

\bibitem[Laskin et~al.(2021)Laskin, Yarats, Liu, Lee, Zhan, Lu, Cang, Pinto,
  and Abbeel]{laskin2021urlb}
Laskin, M., Yarats, D., Liu, H., Lee, K., Zhan, A., Lu, K., Cang, C., Pinto,
  L., and Abbeel, P.
\newblock Urlb: Unsupervised reinforcement learning benchmark.
\newblock \emph{arXiv preprint arXiv:2110.15191}, 2021.

\bibitem[Lee et~al.(2019)Lee, Eysenbach, Parisotto, Xing, Levine, and
  Salakhutdinov]{lee2019efficient}
Lee, L., Eysenbach, B., Parisotto, E., Xing, E., Levine, S., and Salakhutdinov,
  R.
\newblock Efficient exploration via state marginal matching.
\newblock \emph{arXiv preprint arXiv:1906.05274}, 2019.

\bibitem[Lehnert \& Littman(2020)Lehnert and Littman]{lehnert2020successor}
Lehnert, L. and Littman, M.~L.
\newblock Successor features combine elements of model-free and model-based
  reinforcement learning.
\newblock \emph{J. Mach. Learn. Res.}, 21:\penalty0 196--1, 2020.

\bibitem[Lehnert et~al.(2017)Lehnert, Tellex, and
  Littman]{lehnert2017advantages}
Lehnert, L., Tellex, S., and Littman, M.~L.
\newblock Advantages and limitations of using successor features for transfer
  in reinforcement learning.
\newblock \emph{arXiv preprint arXiv:1708.00102}, 2017.

\bibitem[Lehnert et~al.(2020)Lehnert, Littman, and Frank]{lehnert2020reward}
Lehnert, L., Littman, M.~L., and Frank, M.~J.
\newblock Reward-predictive representations generalize across tasks in
  reinforcement learning.
\newblock \emph{PLoS computational biology}, 16\penalty0 (10):\penalty0
  e1008317, 2020.

\bibitem[Lillicrap et~al.(2015)Lillicrap, Hunt, Pritzel, Heess, Erez, Tassa,
  Silver, and Wierstra]{lillicrap2015ddpg}
Lillicrap, T.~P., Hunt, J.~J., Pritzel, A., Heess, N., Erez, T., Tassa, Y.,
  Silver, D., and Wierstra, D.
\newblock Continuous control with deep reinforcement learning.
\newblock \emph{arXiv preprint arXiv:1509.02971}, 2015.

\bibitem[Litman(2005)]{litman2005curiosity}
Litman, J.
\newblock Curiosity and the pleasures of learning: Wanting and liking new
  information.
\newblock \emph{Cognition \& emotion}, 19\penalty0 (6):\penalty0 793--814,
  2005.

\bibitem[Liu et~al.(2021)Liu, Raghunathan, Liang, and Finn]{liu2021decoupling}
Liu, E.~Z., Raghunathan, A., Liang, P., and Finn, C.
\newblock Decoupling exploration and exploitation for meta-reinforcement
  learning without sacrifices.
\newblock In \emph{International conference on machine learning}, pp.\
  6925--6935. PMLR, 2021.

\bibitem[Ma et~al.(2018)Ma, Wen, and Bengio]{ma2018universal}
Ma, C., Wen, J., and Bengio, Y.
\newblock Universal successor representations for transfer reinforcement
  learning.
\newblock \emph{arXiv preprint arXiv:1804.03758}, 2018.

\bibitem[Machado et~al.(2020)Machado, Bellemare, and Bowling]{machado2020count}
Machado, M.~C., Bellemare, M.~G., and Bowling, M.
\newblock Count-based exploration with the successor representation.
\newblock In \emph{Proceedings of the AAAI Conference on Artificial
  Intelligence}, volume~34, pp.\  5125--5133, 2020.

\bibitem[Mutti \& Restelli(2020)Mutti and Restelli]{mutti2020intrinsically}
Mutti, M. and Restelli, M.
\newblock An intrinsically-motivated approach for learning highly exploring and
  fast mixing policies.
\newblock In \emph{Proceedings of the AAAI Conference on Artificial
  Intelligence}, volume~34, pp.\  5232--5239, 2020.

\bibitem[Mutti et~al.(2021)Mutti, Pratissoli, and Restelli]{mutti2021task}
Mutti, M., Pratissoli, L., and Restelli, M.
\newblock Task-agnostic exploration via policy gradient of a non-parametric
  state entropy estimate.
\newblock In \emph{Proceedings of the AAAI Conference on Artificial
  Intelligence}, volume~35, pp.\  9028--9036, 2021.

\bibitem[Mutti et~al.(2022{\natexlab{a}})Mutti, De~Santi, and
  Restelli]{mutti2022maximumentropyexploration}
Mutti, M., De~Santi, R., and Restelli, M.
\newblock The importance of non-markovianity in maximum state entropy
  exploration.
\newblock In Chaudhuri, K., Jegelka, S., Song, L., Szepesvari, C., Niu, G., and
  Sabato, S. (eds.), \emph{Proceedings of the 39th International Conference on
  Machine Learning}, volume 162 of \emph{Proceedings of Machine Learning
  Research}, pp.\  16223--16239. PMLR, 17--23 Jul 2022{\natexlab{a}}.
\newblock URL \url{https://proceedings.mlr.press/v162/mutti22a.html}.

\bibitem[Mutti et~al.(2022{\natexlab{b}})Mutti, Mancassola, and
  Restelli]{mutti2022unsupervised}
Mutti, M., Mancassola, M., and Restelli, M.
\newblock Unsupervised reinforcement learning in multiple environments.
\newblock In \emph{Proceedings of the AAAI Conference on Artificial
  Intelligence}, volume~36, pp.\  7850--7858, 2022{\natexlab{b}}.

\bibitem[Pathak et~al.(2017)Pathak, Agrawal, Efros, and Darrell]{pathak2017icm}
Pathak, D., Agrawal, P., Efros, A.~A., and Darrell, T.
\newblock Curiosity-driven exploration by self-supervised prediction.
\newblock In \emph{Proceedings of the IEEE Conference on Computer Vision and
  Pattern Recognition Workshops}, pp.\  16--17, 2017.

\bibitem[Patil et~al.(2023)Patil, Rahimi-Kalahroudi, Nekoei, Gottipati,
  Samsami, Gupta, Poddar, Zholus, Hashemzadeh, Zhao, and Chandar]{Patil2023}
Patil, D., Rahimi-Kalahroudi, A., Nekoei, H., Gottipati, S.~K., Samsami, M.~R.,
  Gupta, K., Poddar, S., Zholus, A., Hashemzadeh, M., Zhao, X., and Chandar, S.
\newblock Rlhive.
\newblock \url{https://github.com/chandar-lab/RLHive}, 2023.

\bibitem[Pong et~al.(2019)Pong, Dalal, Lin, Nair, Bahl, and
  Levine]{pong2019skew}
Pong, V.~H., Dalal, M., Lin, S., Nair, A., Bahl, S., and Levine, S.
\newblock Skew-fit: State-covering self-supervised reinforcement learning.
\newblock \emph{arXiv preprint arXiv:1903.03698}, 2019.

\bibitem[Rakelly et~al.(2019)Rakelly, Zhou, Finn, Levine, and
  Quillen]{rakelly2019efficient}
Rakelly, K., Zhou, A., Finn, C., Levine, S., and Quillen, D.
\newblock Efficient off-policy meta-reinforcement learning via probabilistic
  context variables.
\newblock In \emph{International conference on machine learning}, pp.\
  5331--5340. PMLR, 2019.

\bibitem[Ross et~al.(2011)Ross, Gordon, and Bagnell]{ross2011reduction}
Ross, S., Gordon, G., and Bagnell, D.
\newblock A reduction of imitation learning and structured prediction to
  no-regret online learning.
\newblock In \emph{Proceedings of the fourteenth international conference on
  artificial intelligence and statistics}, pp.\  627--635. JMLR Workshop and
  Conference Proceedings, 2011.

\bibitem[Schmidhuber(1991)]{schmidhuber1991curiousmodel}
Schmidhuber, J.
\newblock Curious model-building control systems.
\newblock In \emph{[Proceedings] 1991 IEEE International Joint Conference on
  Neural Networks}, pp.\  1458--1463. IEEE, 1991.

\bibitem[Schmidhuber(2009)]{schmidhuber2009simple}
Schmidhuber, J.
\newblock Simple algorithmic theory of subjective beauty, novelty, surprise,
  interestingness, attention, curiosity, creativity, art, science, music,
  jokes.
\newblock \emph{Journal of the Society of Instrument and Control Engineers},
  48\penalty0 (1):\penalty0 21--32, 2009.

\bibitem[Schmidhuber(2010)]{schmidhuber2010formal}
Schmidhuber, J.
\newblock Formal theory of creativity, fun, and intrinsic motivation
  (1990--2010).
\newblock \emph{IEEE transactions on autonomous mental development}, 2\penalty0
  (3):\penalty0 230--247, 2010.

\bibitem[Sekar et~al.(2020)Sekar, Rybkin, Daniilidis, Abbeel, Hafner, and
  Pathak]{sekar2020planning}
Sekar, R., Rybkin, O., Daniilidis, K., Abbeel, P., Hafner, D., and Pathak, D.
\newblock Planning to explore via self-supervised world models.
\newblock In \emph{ICML}, 2020.

\bibitem[Silver et~al.(2014)Silver, Lever, Heess, Degris, Wierstra, and
  Riedmiller]{silver2014deterministic}
Silver, D., Lever, G., Heess, N., Degris, T., Wierstra, D., and Riedmiller, M.
\newblock Deterministic policy gradient algorithms.
\newblock In \emph{International conference on machine learning}, pp.\
  387--395. Pmlr, 2014.

\bibitem[Siriwardhana et~al.(2019)Siriwardhana, Weerasakera, Matthies, and
  Nanayakkara]{siriwardhana2019vusfa}
Siriwardhana, S., Weerasakera, R., Matthies, D.~J., and Nanayakkara, S.
\newblock Vusfa: Variational universal successor features approximator, 2019.

\bibitem[Stadie et~al.(2015)Stadie, Levine, and
  Abbeel]{stadie2015incentivizing}
Stadie, B.~C., Levine, S., and Abbeel, P.
\newblock Incentivizing exploration in reinforcement learning with deep
  predictive models.
\newblock \emph{arXiv preprint arXiv:1507.00814}, 2015.

\bibitem[Strehl \& Littman(2008)Strehl and Littman]{strehl2008analysis}
Strehl, A.~L. and Littman, M.~L.
\newblock An analysis of model-based interval estimation for markov decision
  processes.
\newblock \emph{Journal of Computer and System Sciences}, 74\penalty0
  (8):\penalty0 1309--1331, 2008.

\bibitem[Sutton(1988)]{sutton1988learning}
Sutton, R.~S.
\newblock Learning to predict by the methods of temporal differences.
\newblock \emph{Machine learning}, 3\penalty0 (1):\penalty0 9--44, 1988.

\bibitem[Sutton \& Barto(2018)Sutton and Barto]{sutton2018reinforcement}
Sutton, R.~S. and Barto, A.~G.
\newblock \emph{Reinforcement learning: An introduction}.
\newblock MIT press, 2018.

\bibitem[Tarbouriech \& Lazaric(2019)Tarbouriech and
  Lazaric]{tarbouriech2019active}
Tarbouriech, J. and Lazaric, A.
\newblock Active exploration in markov decision processes.
\newblock In \emph{The 22nd International Conference on Artificial Intelligence
  and Statistics}, pp.\  974--982. PMLR, 2019.

\bibitem[Touati \& Ollivier(2021)Touati and Ollivier]{touati2021learning}
Touati, A. and Ollivier, Y.
\newblock Learning one representation to optimize all rewards.
\newblock \emph{Advances in Neural Information Processing Systems},
  34:\penalty0 13--23, 2021.

\bibitem[Touati et~al.(2022)Touati, Rapin, and Ollivier]{touati2022does}
Touati, A., Rapin, J., and Ollivier, Y.
\newblock Does zero-shot reinforcement learning exist?
\newblock \emph{arXiv preprint arXiv:2209.14935}, 2022.

\bibitem[van Hasselt et~al.(2021)van Hasselt, Madjiheurem, Hessel, Silver,
  Barreto, and Borsa]{vanhasselt2021expectedeligibilitytraces}
van Hasselt, H., Madjiheurem, S., Hessel, M., Silver, D., Barreto, A., and
  Borsa, D.
\newblock Expected eligibility traces.
\newblock In \emph{Proceedings of the AAAI Conference on Artificial
  Intelligence}, volume~35, pp.\  9997--10005, 2021.

\bibitem[Vinyals et~al.(2017)Vinyals, Ewalds, Bartunov, Georgiev, Vezhnevets,
  Yeo, Makhzani, K{\"u}ttler, Agapiou, Schrittwieser,
  et~al.]{vinyals2017starcraft}
Vinyals, O., Ewalds, T., Bartunov, S., Georgiev, P., Vezhnevets, A.~S., Yeo,
  M., Makhzani, A., K{\"u}ttler, H., Agapiou, J., Schrittwieser, J., et~al.
\newblock Starcraft ii: A new challenge for reinforcement learning.
\newblock \emph{arXiv preprint arXiv:1708.04782}, 2017.

\bibitem[Wagener et~al.(2021)Wagener, Boots, and Cheng]{wagener2021safe}
Wagener, N.~C., Boots, B., and Cheng, C.-A.
\newblock Safe reinforcement learning using advantage-based intervention.
\newblock In \emph{International Conference on Machine Learning}, pp.\
  10630--10640. PMLR, 2021.

\bibitem[Yu et~al.(2017)Yu, Gong, Zhong, and Shan]{yu2017unsupervised}
Yu, Y., Gong, Z., Zhong, P., and Shan, J.
\newblock Unsupervised representation learning with deep convolutional neural
  network for remote sensing images.
\newblock In \emph{International conference on image and graphics}, pp.\
  97--108. Springer, 2017.

\bibitem[Zhang et~al.(2017)Zhang, Springenberg, Boedecker, and
  Burgard]{zhang2017deep}
Zhang, J., Springenberg, J.~T., Boedecker, J., and Burgard, W.
\newblock Deep reinforcement learning with successor features for navigation
  across similar environments.
\newblock In \emph{2017 IEEE/RSJ International Conference on Intelligent Robots
  and Systems (IROS)}, pp.\  2371--2378. IEEE, 2017.

\bibitem[Zintgraf et~al.(2019)Zintgraf, Shiarlis, Igl, Schulze, Gal, Hofmann,
  and Whiteson]{zintgraf2019varibad}
Zintgraf, L., Shiarlis, K., Igl, M., Schulze, S., Gal, Y., Hofmann, K., and
  Whiteson, S.
\newblock Varibad: A very good method for bayes-adaptive deep rl via
  meta-learning.
\newblock \emph{arXiv preprint arXiv:1910.08348}, 2019.

\end{thebibliography}

\newpage
\onecolumn
\appendix

\section{Convergence Analysis}
\label{app:convergence}

To gain a deeper understanding why the \alg{} converges to a maximum entropy policy, we consider in this section a simplified dynamic programming variant in Algorithm~\ref{algo:epl-dynamic-programming}.
Note that the \alg{} estimates the SR for a finite CMP for a finite horizon length $h$.
Consequently, the trajectory-action conditioned SR $\psi^\pi(\tau_{:T},a)$ and exploration policy $\pi$ can be stored in an exponentially large but finite look-up table.
Furthermore, with every transition an additional state is appended to the trajectory $\tau_{:T}$, meaning the agent cannot loop back to the same trajectory.
Using these two properties, we state a dynamic programming variant of \alg{} in Algorithm~\ref{algo:epl-dynamic-programming} and then prove its convergence to a policy that maximizes the entropy term $H\left( \pmb{\eta}(\tau_{:T-1}) + \pmb{\psi}^\pi(\tau_{:T},a) \right)$ at every time step.

\begin{algorithm}
  \caption{\alg{}: Dynamic Programming Framework}
  \label{algo:epl-dynamic-programming}
    \begin{algorithmic}[1]
    \FORALL{$\tau_{h}, a$}
        \STATE $\pmb{\psi}^\pi(\tau_{:h},a) \leftarrow \pmb{e}_{ s_{h} }$
    \ENDFOR
    \FOR{$t = h,...,2$}
        \FORALL{$\tau_{:t}, a$}
            \STATE $\pi(\tau_{:t}) \leftarrow \arg \max_{a} H\left( \pmb{\eta}(\tau_{:t-1}) + \pmb{\psi}^{\pi}(\tau_{:t},a) \right)$
            \STATE $\pmb{\psi}^\pi(\tau_{:t-1},a) \leftarrow \pmb{e}_{s_{t-1}} + \gamma(t) \pmb{\psi}^\pi(\tau_{:t}, \pi(\tau_{:t}) )$
        \ENDFOR
    \ENDFOR
    \STATE \textbf{return} $\pi$ such that $\pi(\tau) = \arg \max_a H( \pmb{\eta}(\tau_{:-1}) + \pmb{\psi}(\tau,a) )$.
    \end{algorithmic}
\end{algorithm}

The convergence proof uses the following property of the predecessor trace $\pmb{\eta}$ and SR $\pmb{\psi}^\pi$:
Consider a trajectory $\tau$ which selects action $a_T$ at time step $T$,
then
\begin{align}
    \pmb{\eta}(\tau_{:T-1}) + \pmb{\psi}^\pi( \tau_{:T},a_T ) &= \underbrace{\sum_{t=1}^{T-1} \gamma(t) \pmb{e}_{s_t}}_{=\pmb{\eta}(\tau_{:T-1})} + \underbrace{\mathbb{E}_{\tau_{T+1:}, \pi} \Bigg[  \sum_{t=T}^h\gamma(t) \pmb{e}_{s_t} \Bigg| \tau_{:T},a_T \Bigg]}_{=\pmb{\psi}^\pi( \tau_{:T}, a_T )} &\text{(by Eq.~\eqref{eq:objective_sf_pf})} \nonumber \\
    &= \sum_{t=1}^{T-1} \gamma(t) \pmb{e}_{s_t} + \mathbb{E}_{\tau_{T+1:}, \pi} \Bigg[ \gamma(T)\pmb{e}_{s_T} + \sum_{t=T+1}^h\gamma(t) \pmb{e}_{s_t} \Bigg| \tau_{:T},a_T \Bigg] \\
    &= \sum_{t=1}^{T} \gamma(t) \pmb{e}_{s_t} + \mathbb{E}_{\tau_{T+1:}, \pi} \Bigg[ \sum_{t=T+1}^h\gamma(t) \pmb{e}_{s_t} \Bigg| \tau_{:T},a_T \Bigg] \\
    &= \pmb{\eta}(\tau_{:T}) + \mathbb{E}_{\tau_{T+1:}, \pi} \left[ \pmb{\psi}^\pi( \tau_{:T+1}, \pi(\tau_{:T+1}) ) \middle| \tau_{:T},a_T \right] \label{eq:eta-psi-splitting}
\end{align}
Using this identity, we can prove the convergence of Algorithm~\ref{algo:epl-dynamic-programming}.

\begin{proposition}
    The policy $\pi^*$ returned by Algorithm~\ref{algo:epl-dynamic-programming} is such that for every $t$-step trajectory $\tau_{:t}$ where $t \le h$, 
    \begin{equation}
        \pi^*(\tau_{:t}) \in \arg \max_\pi H( \pmb{\eta}(\tau_{:t-1}) + \pmb{\psi}^\pi(\tau_{:t}, \pi(\tau_{:t})) ).
    \end{equation}
\end{proposition}
\begin{proof}
    The proof proceeds by induction on the length of an $h$-step trajectory, starting with a length of $h$ and iterating to a length of one.

    \textbf{Induction hypothesis:} 
    We define a sub-sequence optimal policy $\pi_t$ such that for every $k$-step trajectory prefix $\tau_{:k}$ and $t\le k \le h$,
    \begin{equation}
        \pi_t \in \arg \max_\pi H( \pmb{\eta}(\tau_{:{k-1}}) + \pmb{\psi}^{\pi}(\tau_{:k},\pi( \tau_{:k} )) ). \label{eq:induction-hypothesis}
    \end{equation}
    The exploration policy $\pi_t$ is the optimal after executing the first $t$ steps of an $h$-step trajectory $\tau$.
    The goal is to prove that the induction hypothesis in line~\eqref{eq:induction-hypothesis} holds for $t=1$.

    \textbf{Base case:} 
    The base case for $t=h$ holds trivially, because SR does not have a dependency on the policy $\pi$ for an $h$-step trajectory.
    Therefore the policy $\pi$ can output any action for a trajectory sequence of length $h$:
    \begin{equation}
        \max_\pi H( \pmb{\eta}(\tau_{:h-1}) + \pmb{\psi}^\pi(\tau_{:h}, \pi(\tau_{:h})) ) = H( \pmb{\eta}(\tau_{:h-1}) + \gamma(h) \pmb{e}_{s_{h}} ).
    \end{equation}

    \textbf{Induction Step:}
    Suppose the induction hypothesis in line~\eqref{eq:induction-hypothesis} holds for some $t > 1$ and $\pi_t$ is the maximizer of 
    \begin{equation}
        H( \pmb{\eta}(\tau_{:{k-1}}) + \pmb{\psi}^{\pi_t}(\tau_{:k},\pi_t( \tau_{:k} )) ) 
    \end{equation}
    where $t \le k \le h$.
    For time step $t-1$, we have that for some action $a$,
    \begin{align}
        H( \pmb{\eta}(\tau_{:t-2}) + \pmb{\psi}^{\pi_{t}}(\tau_{:t-1}, a) ) 
        &= H( \pmb{\eta}(\tau_{:t-2}) + \gamma(t-1) \pmb{e}_{s_{t-1}} + \mathbb{E} \left[ \pmb{\psi}^{\pi_{t}}(\tau_{:t}, \pi_t(\tau_{:t}) \middle| s_{t-1},a \right] ) \\
        &= H( \pmb{\eta}(\tau_{:t-1}) + \mathbb{E} \left[ \pmb{\psi}^{\pi_{t}}(\tau_{:t}, \pi_t(\tau_{:t})) \middle| s_{t-1},a \right] ) \\
        &= H( \mathbb{E} \left[ \pmb{\eta}(\tau_{:t-1}) + \pmb{\psi}^{\pi_{t}}(\tau_{:t}, \pi_t(\tau_{:t})) \middle| s_{t-1},a \right] ).
    \end{align}
    We note that the term inside the expectation is already maximized by $\pi_t$ (by induction hypothesis).
    If we now set $\pi_{t-1}$ to be equal to $\pi_t$ for every $t$-step or longer trajectory and set
    \begin{equation}
        \pi_{t-1}(\tau_{:t-1}) = \arg \max_a H( \pmb{\eta}(\tau_{:t-2}) + \pmb{\psi}^{\pi_{t}}(\tau_{:t-1}, a) ),
    \end{equation}
    then for $t - 1 \le k \le h$ 
    \begin{equation}
        \pi_{t-1} \in \arg \max_\pi H( \pmb{\eta}(\tau_{:{k-1}}) + \pmb{\psi}^{\pi}(\tau_{:k},\pi( \tau_{:k} )) ).
    \end{equation}
    This completes the proof.
\end{proof}

\section{\alg{}- Policy Gradient}
\label{app:policy_gradient}

Application of Q-Learning based approaches to continuous action space is not easy because finding the greedy action at any time step can be slow to be practical with large, unconstrained function approximators and nontrivial action spaces. 
In this work, we take a similar approach to deterministic policy gradient~\cite{silver2014deterministic} to learn exploratory policies.
The objective remains the same which is to maximize the entropy of state visitation distribution.
However, it is challenging to estimate the gradient where the objective is based on the entropy term. 
Previous works have either used alternate optimization~\cite{pong2019skew,lee2019efficient} or similar objective functions~\cite{guo2021geometric}. 
The challenge is because of the expectation inside the logarithm in \autoref{eq:entropy-objective-pi}. 
\cite{lee2019efficient,pong2019skew} addressed this intractability by first estimating the visited state distribution and then using this estimate to optimize the entropy-based objective.
Unfortunately, such alternating approaches are often are prone to instability and slow convergence~\cite{guo2021geometric}.
In this work, we take an alternative direction and learn a network to directly estimate the visited state distribution.
The combination of predecessor trace $\pmb{\eta}$ and successor representation $\pmb{\psi}^{\pi}$ can be leveraged to estimate the state visitation distribution which is obtained using:
\begin{align}
    \pmb{\eta}(\tau_{:T-1}) + \pmb{\psi}^\pi( \tau_{:T},a_T ) &= \underbrace{\sum_{t=1}^{T-1} \gamma(t) \pmb{e}_{s_t}}_{=\pmb{\eta}(\tau_{:T-1})} + \underbrace{\mathbb{E}_{\tau_{T+1:}, \pi} \Bigg[  \sum_{t=T}^h\gamma(t) \pmb{e}_{s_t} \Bigg| \tau_{:T},a_T \Bigg]}_{=\pmb{\psi}^\pi( \tau_{:T}, a_T )} &\text{(by Eq.~\eqref{eq:objective_sf_pf})} \nonumber \\
\end{align}
The \gls{sr} vector can be learned with gradient based optimization and provides the estimate of state visitation distribution for a given policy $\pi$ and trajectory. 
The policy can utilize this estimate to learn optimal behaviors for efficient exploration in the environment.

To learn optimal behaviors for continuous action spaces, \alg{} uses an actor-critic architecture comprising of a deterministic actor~$\pi_{\mu}(\tau)$ that provides the action and a critic to estimate the utility function. 
Here, both the actor and critic networks are non-Markovian and depend on the entire history of visited states.
The goal of the critic network is to approximate the Q-function for a given trajectory $\tau$ and a given action $a \in \mathcal{A}$.
For a given trajectory $\tau_{:T}$, critic computes this by combining the predecessor representation and the \gls{sr} vector.
The predecessor representation is fixed for a given history, implying that the critic only needs to approximate the \gls{sr}~$\pmb{\psi}_{\theta}(\tau, a)$.
To summarize, the critic estimates the Q-function as shown below:
\begin{align}
    \label{eq:Q-function_entropy_}
    Q_{\theta,\text{expl}}(\tau_{:T}, a_T) = H\left(\pmb{\eta} (\tau_{:T-1}) + \pmb{\psi}_{\theta} (\tau_{:T}, a_T)\right),
\end{align}
To update the critic network, we update the \gls{sr} approximator network using temporal-difference error. 
The target for the \gls{sr} is obtained using the action coming from the current policy~~$a'_{T+1}=\pi_{\mu}(\tau_{:T+1})$, and is given by
\begin{align}
    \pmb{y} = \pmb{e}_{s_T} + \gamma(T+1)\pmb{\psi}_{\theta}(\tau_{:T+1}, a'_{T+1}).
\end{align}
The \gls{sr} network is updated with gradient-based learning to optimize the Mean-Squared Error between the predicted \gls{sr} and the target, and the loss function $\mathcal{L}_{SR}$ is given by
\begin{align}
    \label{eq:loss_sr}
    \mathcal{L}_{SR} &=  || \pmb{\psi}_{\pmb{\theta}} (\tau_{:T}, a_{T})  - \pmb{y}(\tau_{:T+1},a'_{T+1}). ||^2 &\text{(by Eq.~\eqref{eq:loss_sr})} \nonumber
\end{align}
Given an estimate of the \gls{sr} for the current policy, we need a mechanism to update the actor network to maximize the objective.
Deterministic policy gradient algorithm~\cite{silver2014deterministic} provided a way of learning optimal policies with a deterministic actor. 
In this work, we formulate the gradient for the actor parameters using similar mechanism with the goal to maximize the entropy-based utility function.
Proposition~\autoref{prop:pg} presents a derivation of the gradients for the actor network parameters obtained by applying the chain rule on the Shannon-entropy based Q-function. 
\begin{proposition}\label{prop:pg}
    Assuming the CMP satisfies~\citep[conditions A.1]{silver2014deterministic} (all functions are continuous and differentiable across all parameters) and for a $\mu$-parameterized policy function $\pi_{\mu}$ the gradient with respect to $\mu$ of the maximum entropy objective
    \begin{equation*}
        J(\pi_{\mu}) = \mathbb{E}_{\tau\sim\rho} [H(\pmb{\eta}(\tau_{:-1}) + \pmb{\psi} (\tau, \pi_{\mu}(\tau))]
    \end{equation*}
    is
    \begin{equation*}
        \nabla_{\mu} J(\pi_{\mu}) = \mathbb{E}_{\tau\sim\rho} \Big[ \sum_i z_i \nabla_{\mu} \pi_{\mu} (\tau) \nabla_a \pmb{\psi}_i(\tau,a) \big|_{a=\pi_{\mu}(\tau)}\Big] .
    \end{equation*}
    where $z_i=-\log[\pmb{\eta}(\tau_{:-1})_i + \pmb{\psi}(\tau, \pi_{\mu}(\tau))_i] - 1$, $H$ is the Shannon-Entropy function over the representation vectors, and the expectation over trajectories is computed with respect to some trajectory visitation distribution $\rho$. 
\end{proposition}
\begin{proof}

We begin by rewriting the Shannon Entropy here for a $T$-step trajectory as 
\begin{align}
    H(\pmb{\eta}(\tau_{:T-1}) + \pmb{\psi} (\tau_{:T}, a_T)) &= -\sum_i (\pmb{\eta}(\tau_{:T-1})_i + \pmb{\psi} (\tau_{:T}, a_T)_i) \log ((\pmb{\eta}(\tau_{:T-1})_i + \pmb{\psi} (\tau_{:T}, a_T)_i),
\end{align}
where $a_T = \pi_\mu(\tau_{:T})$.

To simplify the notations, we will use $\pmb{\eta}_i=\pmb{\eta}(\tau_{:T-1})_i$ and $\pmb{\psi}_i=\pmb{\psi} (\tau_{:T}, a_T)_i$ to represent the $i$th term of the predecessor and successor representation vectors.
Now taking the gradient with respect to the actor parameters $\mu$ gives:
\begin{align}
    \nabla_{\mu} H(\pmb{\eta}_i + \pmb{\psi}_i) &= -\nabla_{\mu} \sum_i (\pmb{\eta}_i + \pmb{\psi}_i) \log (\pmb{\eta}_i + \pmb{\psi}_i) \\
    &= -\sum_i [\nabla_{\mu} (\pmb{\eta}_i + \pmb{\psi}_i) \log(\pmb{\eta}_i + \pmb{\psi}_i)] \\
    &= -\sum_i [\log(\pmb{\eta}_i + \pmb{\psi}_i) \nabla_{\mu} (\pmb{\eta}_i + \pmb{\psi}_i) +  (\pmb{\eta}_i + \pmb{\psi}_i) \nabla_{\mu} \log(\pmb{\eta}_i + \pmb{\psi}_i)] \\
    &= -\sum_i [\log(\pmb{\eta}_i + \pmb{\psi}_i) \nabla_{\mu} \pmb{\psi}_i +
     \nabla_{\mu} \pmb{\psi}_i \\
     &= -\sum_i [\log(\pmb{\eta}_i + \pmb{\psi}_i) + 1] \nabla_{\mu} \pmb{\psi}_i 
\end{align}

Now, using the chain rule on the $i$th feature in \gls{sr}, we obtain 
\begin{align}
    \nabla_{\mu} \pmb{\psi}_i = \nabla_{\mu} \pi_{\mu}(\tau_{:T}) \nabla_a \pmb{\psi}(\tau_{:T},a)_i |_{a=\pi_{\mu}(\tau_{:T})} 
\end{align}

By substitution
\begin{equation}
    \nabla_{\mu} H(\pmb{\eta}_i + \pmb{\psi}_i) = -\sum_i [\log(\pmb{\eta}_i + \pmb{\psi}_i) + 1] \nabla_{\mu} \pi_{\mu}(\tau_{:T}) \nabla_a \pmb{\psi}(\tau_{:T},a)_i |_{a=\pi_{\mu}(\tau_{:T})}.
\end{equation}

Therefore, the gradient of the overall objective is
\begin{align}
    \nabla_{\mu}J(\pi_{\mu}) &= \mathbb{E}_{\tau\sim\rho} [ \nabla_{\mu} H(\pmb{\eta}(\tau_{:-1}) + \pmb{\psi} (\tau, \pi_{\mu}(\tau))] \\
    &= - \mathbb{E}_{\tau\sim\rho} \left[\sum_i [\log(\pmb{\eta}_i + \pmb{\psi}_i) + 1] \nabla_{\mu} \pi_{\mu}(\tau_{:T}) \nabla_a \pmb{\psi}(\tau_{:T},a)_i |_{a=\pi_{\mu}(\tau_{:T})} \right].
\end{align}

This completes the proof.
\end{proof}
In Proposition~\ref{prop:pg}, we derive the gradient of the actor parameters for the maximum state entropy exploration objective.
Taking inspiration from algorithms~\citep{lillicrap2015ddpg,fujimoto2018td3,haarnoja2018soft} that extend \gls{dpg} to make the optimization process stable when scaling to larger state and action space, we base our implementation to be similar to the TD3~\citep{fujimoto2018td3} algorithm.
In \autoref{app:architecture}, we outline the learning procedure to learn using the policy gradient derived in Proposition~\ref{prop:pg}. 
Furthermore, we also discuss how the proposed algorithm handles continuous state spaces.

\section{Neural Network Architecture}
\label{app:architecture}

~\alg{} approximates the \gls{sr} with a parameterized function $\pmb{\psi}_{\pmb{\theta}}$ to learn an exploration policy and predict the state visitation distribution.
Because the \gls{sr} is conditioned on a trajectory $\tau$ of variable length, we implement the function $\pmb{\psi}_{\pmb{\theta}}$ with a \glsfirst{rnn} architecture, as outlined in~\autoref{fig:network_architecture}.
In this architecture, the states in a trajectory $\tau_{:T}$ are first fed through a encoder network (E) comprising of \glsfirst{mlp} layers.
Subsequently, the output of the \gls{mlp} is fed through an \gls{rnn} (denoted with F) architecture to compress the state sequence into one real-valued feature vector.
Since, \gls{rnn} are known to suffer from vanishing gradients~\cite{bengio1994learning}, we implement the \gls{rnn} with a \gls{gru}~\cite{cho2014gru}. 
Leveraging recurrent networks to learn the \gls{sr} has been explored previously in \citep{barreto2018transfer, borsa2018universal}. 
Finally, the recurrent state obtained from the \gls{rnn} is concatenated with the representation of the current state and is passed through the the decoder (D) with \gls{mlp} layers to predict an \gls{sr} vector for a given action. 
In the following paragraphs, we elaborate on how the proposed architecture was used to train the agent for finite and infinite action spaces.

\begin{algorithm}[tb]
   \caption{\alg{}: Finite Action Space Framework}
   \label{algo:epl_qlearning}
\begin{algorithmic}[1]
    \STATE Initialize \gls{sr} network with parameters $\theta$ and the replay buffer $\mathcal{B}=\{\}$
    \STATE Denote the predecessor feature with $\pmb{\eta}$, discount function with $\gamma$, and episode length with $h$
    
    \WHILE{Training}
    \STATE Collect $\tau_{exp}=\{s_1,a_1,..,s_h\}$ using current policy $\pi_{\theta}$ and add it to replay buffer $\mathcal{B}$=$\mathcal{B} \cup \tau_{exp}$
    \\
    \FOR{each training step}
        \STATE Sample batch of $\tau=(s_1,..a_{l-1},s_l)$ $\sim$ $\mathcal{B}$ of sequence length $l \in \{2,..,h\}$
        \STATE Compute $a'$= $\arg \max_{a \in \mathcal{A}} H(\pmb{\eta} (\tau) + \pmb{\psi}_{\theta} (\tau, a))$
        \STATE Compute target $\pmb{y} = \pmb{e}_{s_{l-1}} + \gamma(l) \pmb{\psi}_{\theta}(\tau, a')$
        \STATE Update SR network by performing gradient step on $||\pmb{y}$ - $\pmb{\psi}_{\theta}(\tau_{:l-1}, a_{l-1})||_2^2$
    \ENDFOR
   \ENDWHILE
\end{algorithmic}
\end{algorithm}

\paragraph{Finite Action Space Variant}
For the finite action space variant, the decoder outputs a \gls{sr} vector for each action $a\in\mathcal{A}$. This is similar to prior method that learn \gls{sf} for discrete action spaces~\citep{lehnert2017advantages,barreto2017successor}. Algorithm~\ref{algo:epl_qlearning} describes the learning procedure for training \alg{} to get exploratory policies.
\begin{algorithm}[tb]
\caption{\alg{}: Infinite Action Space Framework}
\label{algo:epl_pg}
\begin{algorithmic}[1]
    \STATE Initialize \gls{sr} network with parameters $\theta_1$, $\theta_2$, policy parameters $\mu$ and the replay buffer $\mathcal{B}=\{\}$
    \STATE Set target parameters equal to the main parameters: $\theta_{targ,1}=\theta_1$, $\theta_{targ,2}=\theta_2$, and $\mu_{targ} \leftarrow \mu$
    \STATE Denote the predecessor feature with $\pmb{\eta}$, discount function with $\gamma$, and episode length with $h$
    \WHILE{Training}
        \STATE Collect $\tau_{exp}=\{s_1,a_1,..,s_h\}$ using target policy $\pi_{\mu_{targ}}$ and add it to replay buffer $\mathcal{B}$=$\mathcal{B} \cup \tau_{exp}$
        \\
        \FOR{each training step $j$}
            \STATE Sample batch of $\tau=(s_1,..a_{l-1},s_l)$ $\sim$ $\mathcal{B}$ of sequence length $l \in \{2,..,h\}$
            \STATE Compute target actions $a' = clip(\pi_{\mu_{targ}}(\tau_{:l}) + clip(\epsilon, -c, c), a_{Low}, a_{High}), \epsilon\sim\mathcal{N}(0, 1)$
            \STATE Compute i=$\arg\min_{k \in \{ 1,2 \}}$ $H$($\pmb{\eta}(\tau_{:l-1})$ + $\pmb{\psi}_{\theta_k}(\tau_{:l-1}, a'))$ 
            \STATE Compute target $\pmb{y}$=$\pmb{e}_{s_l}$ + $\gamma(l)$ $\pmb{\psi}_{\theta_i}(\tau_{:l-1}, a')$
            \STATE Update the \gls{sr} networks by performing gradient steps on\\ ~~~~~~~~~~~~~~~~~~~$||\pmb{y}$ - $\pmb{\psi}_{\theta_i}(\tau_{:l-1}, a_{l-1})||_2^2$,~~~~~~~~~~ $i=1,2$
            \IF{j \% policy\_update == 0}
                \STATE Perform update step for policy by computing gradients using\\ ~~~~~~~~~~~~~~~~~~~$\sum_{i}$ $z_i$ $\nabla_a$ $\pmb{\psi}_{\theta_1}(\tau_{:l},a)$ $|_{a=\pi(\tau_{:l})}$ $\nabla_{\mu}$ $\pi_{\mu}(\tau_{:l})$,
                \\
                where $z_i = -\log(\pmb{\eta}(\tau_{:l})_i$ - $\pmb{\psi}_{\theta_1}(\tau_{:l}, \pi_{\mu}(\tau_{:l}))_i) + 1$
                \STATE Update target networks with \\
                ~~~~~~~~~~~~$\theta_{targ,i}$ $\leftarrow$ $\rho\theta_{targ,i} + (1 - \rho) \theta_{i}$,~~~~~~~~~~ $i=1,2$ \\
                ~~~~~~~~~~~~$\mu_{targ} \leftarrow \rho\mu_{targ} + (1 - \rho) \mu$
            \ENDIF
        \ENDFOR
    \ENDWHILE
\end{algorithmic}
\end{algorithm}

\paragraph{Infinite Action Space Variant}
For infinite action space variant, the hidden state from the recurrent network is passed through a deterministic actor network which comprises of \gls{mlp} layers.
The policy network (actor) is conditioned on the hidden states because in \alg{} the policy is a function of trajectories and not individual states.
The hidden state from the recurrent network is concatenated with the action to predict the \gls{sr} vector.
The estimated \gls{sr} vector is used to calculate the visitation distribution over one-hot embeddings of states and these \gls{sr} predictions are then used to computed to loss objective for optimization.
For the Reacher and Pusher tasks, we manually sub-select which dimensions of the state space are one-hot encoded.
In these cases, \alg{} learns an exploration policy that maximizes the entropy of visitation distribution across these sub-selected dimensions only.
This approach to sub-selecting state dimensions is similar to prior work on maximum state entropy exploration~\citep{pmlr-v97-hazan19a,mutti2022maximumentropyexploration}.
In this work, the agent is trained using similar techniques as the TD3 agent~\citep{fujimoto2018td3}.
The agent keeps a single encoder and recurrent network to encode the history of observed states.
The encoded states are passed through two decoder networks to predict the \gls{sr} vectors, which are used to represent the two critic networks.
The target for \gls{sr} is computed using the vector that leads to a smaller value of the two utility functions.
There is a single actor network that specifies the action from the hidden state.
In addition, \alg{}-maintains a target network for all the components---encoder, recurrent, critics, and actor networks.
Furthermore, similar to the TD3 algorithm, a clipped noise is added to each dimension of the action from the target network.
Moreover, we also use delayed actor updates where the actor network is updated less frequently than the \gls{sr} networks.
Lastly, the gradients from the actor are not passed through the encoder and the recurrent networks.
The procedure for training this variant is provided in Algorithm~\ref{algo:epl_pg}. 

\begin{figure}[ht]
    \centering
    \includegraphics[width=0.9\columnwidth]{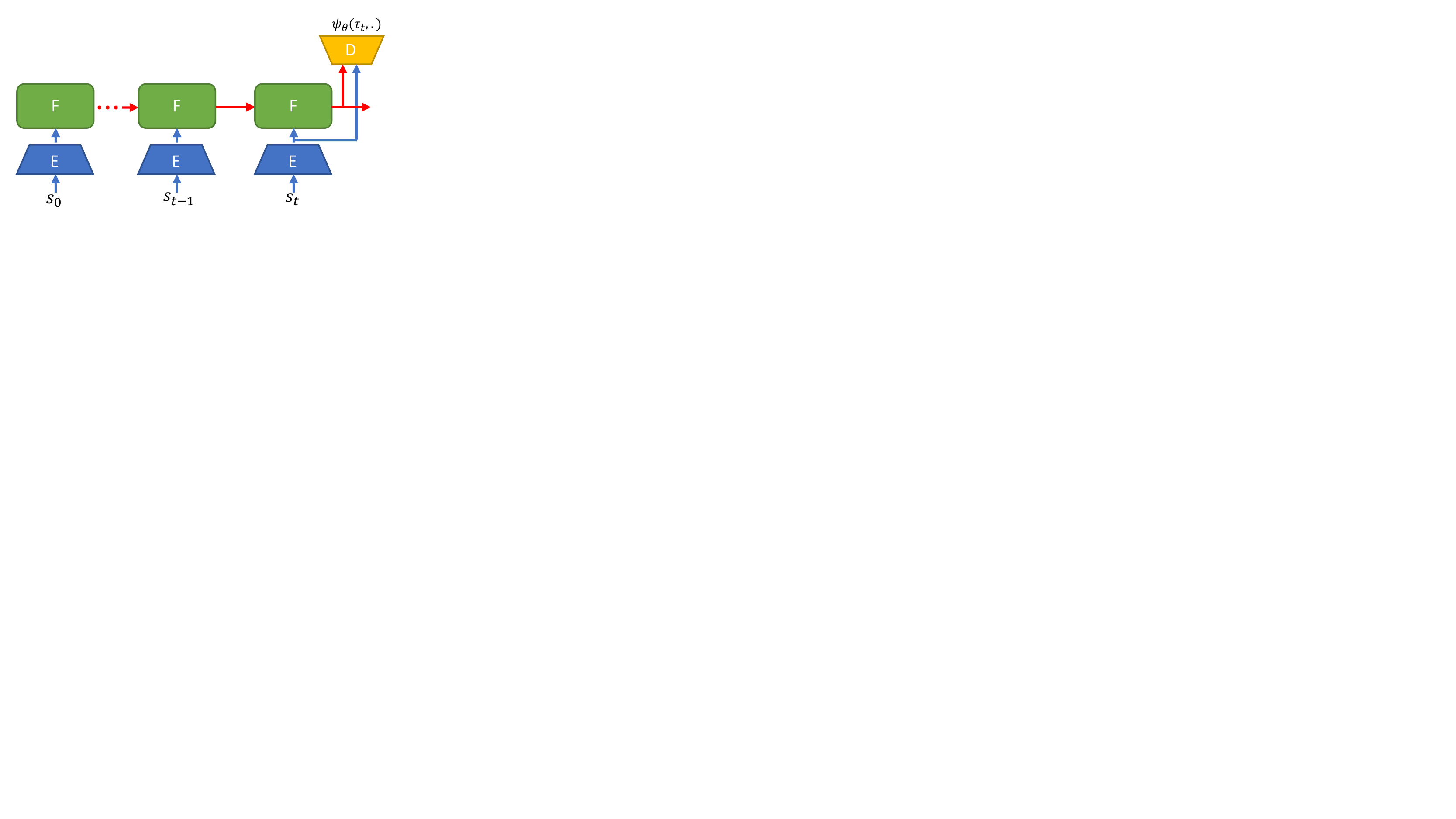}
    \caption{Network architecture to learn the \gls{sr}. The states are firstly passed through an encoder (E), followed by feeding the encoded states through a \gls{rnn} (\gls{gru} in our case) (F). This compresses the history of visited states, and the obtained hidden state is concatenated with the encoded state to predict the \gls{sr} vector for an action using a decoder network (D).}
    \label{fig:network_architecture}
\end{figure}

\section{Discount Function}
\label{app:discount_function}
\begin{figure}[ht]
    \centering
    \includegraphics[width=0.5\textwidth]{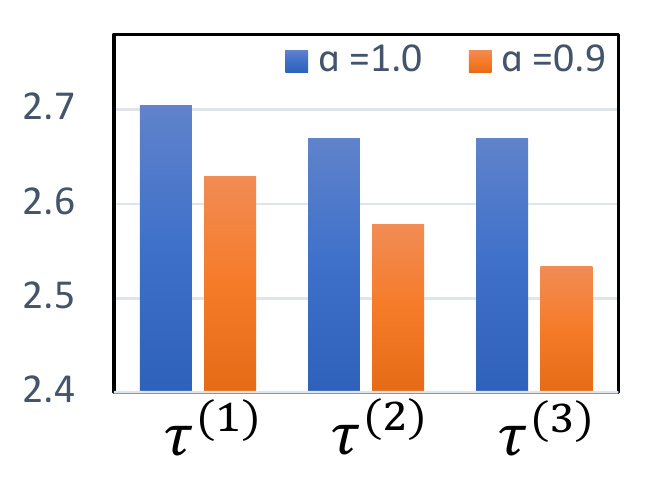}
    \caption{Illustration of the values of the entropy with different value of $\alpha$ hyperparameter in the proposed $\gamma$-function for the trajectories introduced in \autoref{fig:motivation}.}
    \label{fig:ablation_alpha}
\end{figure}
In this work, we have used a time-dependent $\gamma$-function. 
Using the gridworld example described in \autoref{fig:motivation}, we now present how the choice of $\gamma$-function affects the entropy term in the objective. Suppose there are three trajectories followed by the given trace $\tau_{:T}$, where we denote the $i$-th trajectory with $\tau^{(i)}$. 
Here, Figure~\ref{fig:motivation_traj1} shows an optimal trajectory~($\tau^{(1)}$) which combined with the trace covers each cell of the grid with $15$ steps.
Figure~\ref{fig:motivation_traj2} presents a suboptimal trajectory~($\tau^{(2)}$) where the agent takes the right action from the current state and visits a previously observed state in the last step. Figure~\ref{fig:motivation_traj3} shows another sub-optimal trajectory~($\tau^{(3)}$) which takes the right action in the current state but visits the new state twice because it goes to the top right corner of the grid.

For the intermediate step $T$, we define the discount factor for the predecessor representation for the trace as $\gamma(t) = \frac{\alpha^{T - t}}{Z}$, where $\alpha$ is a scalar between (0, 1], and $Z=\Sigma_{t=0}^T \alpha^{T-t} + \Sigma_{t=T}^h \alpha^{t-T}$ is the normalization factor.
The $\gamma$-function for the successor representation is denoted using $\gamma(t) = \frac{\alpha^{t - T}}{Z}$. 
The proposed $\gamma$-function for both the representations is similar to discounting used in standard RL literature~\cite{dayan1993improving,sutton1988learning}.
Upon comparing the entropy for given trajectories with $\alpha=0.9,1.0$, we observe in ~\autoref{fig:ablation_alpha} that $\tau^{(1)}$ being the optimal trajectory attains higher entropy when compared with $\tau^{(2)}$ and $\tau^{(3)}$. 
The other sub-optimal trajectories $\tau^{(2)}$ and $\tau^{(3)}$ achieve same entropy when $\alpha$ is set to 1.0.
However, for $\alpha=0.9$, the discount function $\gamma$ emphasizes which states are visited earlier in the trajectory and assigns the lowest score to the trajectory $\tau^{(3)}$ because this trajectory revisits states earlier in the sequence than the other options $\tau^{(2)}$ and $\tau^{(1)}$.
This example illustrates how the $\gamma$-function can be used to trade off near-term vs. long-term exploration behavior.
Depending on the $\alpha$ setting, the agent can be encouraged to avoid re-visiting states either only in the short-term or the long-term, similar to how discounting encourages maximizing short-term over long-term rewards in algorithms like Q-learning~\cite{watkins1992qlearning}.

\section{Environments}
\label{app:environments}
\begin{figure*}[t]
\centering
\subfigure[]{
\includegraphics[width=.45\textwidth]{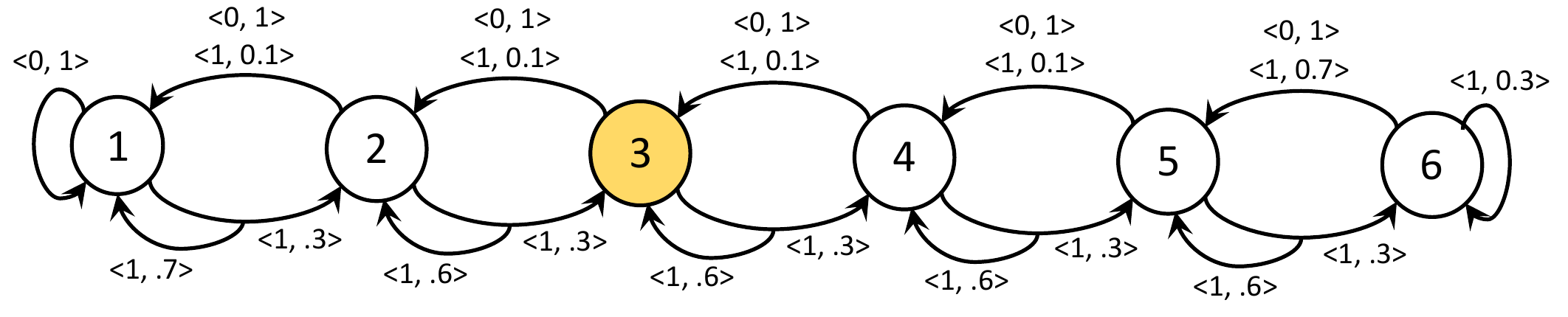}
\label{fig:rivewswim}
}
\subfigure[]{
\includegraphics[width=.25\textwidth]{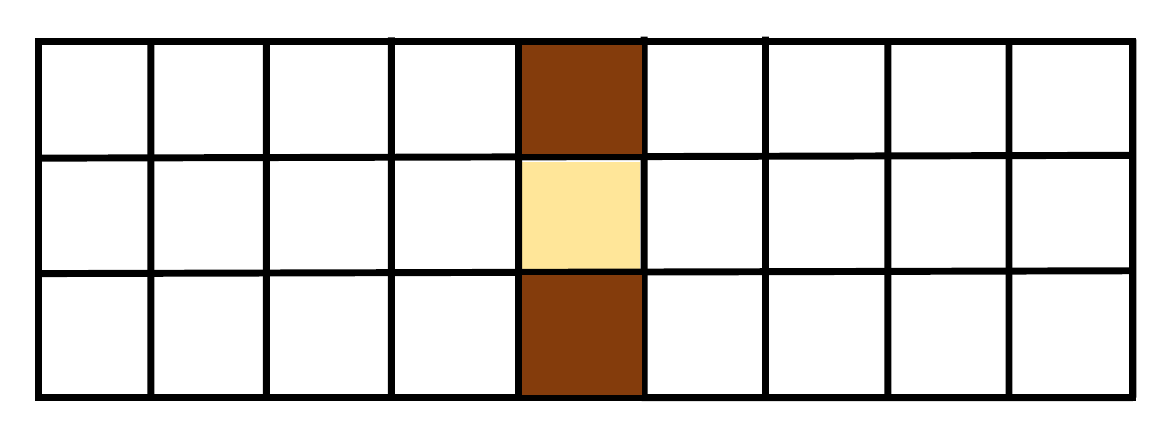}
\label{fig:tworooms}
}
\subfigure[]{
\includegraphics[width=.22\textwidth]{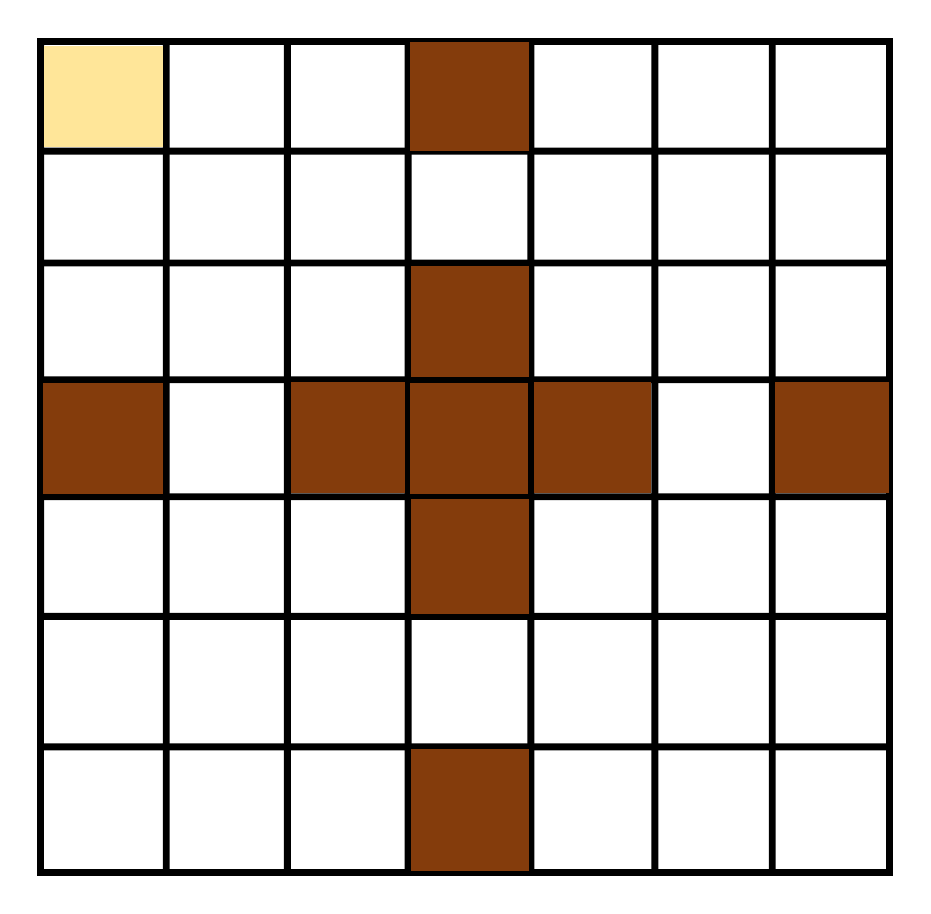}
\label{fig:fourrooms}
}
\caption{(a) RiverSwim~\cite{strehl2008analysis}, (b) TwoRooms, and (c) FourRooms environments, respectively. The yellow block denote the initial state of the agent in each episode. The brick red regions represent the walls in the environment. When taking an action that collides with the walls, the chosen action does not change the state of the agent.}
\label{fig:environment_finite}
\end{figure*}

\begin{figure*}[t]
\centering
\subfigure[]{
\includegraphics[width=.3\textwidth]{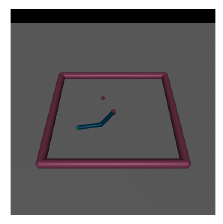}
\label{fig:reacher}
}
\subfigure[]{
\includegraphics[width=.3\textwidth]{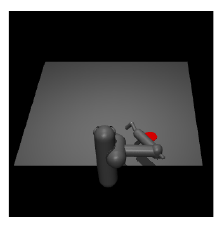}
\label{fig:pusher}
}
\caption{(a) Reacher and (b) Pusher environments for experiments with infinite action space.}
\label{fig:environment_infinite}
\end{figure*}

For finite action space variant, we experimented with ChainMDP, RiverSwim, $5\times5$ Grid-world, TwoRooms and FourRooms environments, which have finite action and state space, respectively. For infinite action space variant, we experiment with Reacher and Pusher environments where we want the agent to move its fingertip to different locations in the environment. The environments used in this work are further described below:

\paragraph{ChainMDP:} ChainMDP is an environment where the agent can take only move in two directions---left or right. In this work, we experiment with both deterministic and stochastic version of ChainMDP environments. The stochastic ChainMDP is similar to RiverSwim environment~\cite{strehl2008analysis}~(Figure~\ref{fig:rivewswim}).

\paragraph{GridWorld:} In the Gridworld environment~(shown in Figure~\ref{fig:motivation}), the agent can take 4 actions to move in any of the 4 directions. In this work, we experiment with the gridworld of dimensions~$5\times5$. The agent always start in the top left corner of the grid. For the gridworld, there are multiple possible optimal trajectories, and the number of such trajectories increases explonentially with size of grid. 

\paragraph{TwoRooms:} The proposed TwoRooms environment is a gridworld with some walls. As shown in Figure~\ref{fig:tworooms}, the agent starts at the center of wall between the two rooms and has to first navigate in one of the rooms, visit the starting state and then move to the other room. This makes the task challenging as the agents requires to track the trace because when the agents reaches initial state after exploring one room, the information of which room was visited should aid in going to the other room.

\paragraph{FourRooms:}
The FourRooms environment~(depicted in Figure~\ref{fig:fourrooms}) has four rooms connected which are connected by open shots between the walls. This task is even more challenging as the agents while navigating need to first explore the current room followed by efficiently going across different rooms. 

\paragraph{Reacher:}
The Reacher environment~(shown in Figure~\ref{fig:reacher}) is a continuous control environment having a two-jointed robotic arm with continuous state and action spaces. The action space denotes the torques applied to the hinges. 
The state denotes the position, angles and angular velocities of the arms. 
The agent is tasked to maximize the entropy over the position of the fingertip.

\paragraph{Pusher:}
The Pusher environment~(shown in Figure~\ref{fig:pusher}) is a continuous control environment having a multiple-jointed robotic arm with continuous state and action spaces. The action space denotes the torques applied to the hinges. 
The state denotes the position, angles and angular velocities of the arms/hinges. 
Similar to the Reacher environment, the agent is tasked to maximize the entropy over the position of the fingertip.
However, this task is harder to solve because having multiple joints leads to a larger action and state space making it a more challenging control problem.

For training and evaluation, we have used different parameters specific to each of the environment. For training, the parameters are the length of an episode and the number of episodes used for training. The number of environment steps can be obtained by multiplying these 2 parameters. For evaluation of an agent, we use different episode length for the three metrics defined to measure the performance of agents in \autoref{sec:experiments}.

\begin{table*}[hbt!]
\scriptsize
	\centering
		\begin{tabular}{l|ccccc}
			\toprule
			\textbf{Name} & \textbf{ChainMDP} & \textbf{RiverSwim} & \textbf{Gridworld} & \textbf{TwoRooms} & \textbf{FourRooms}\\
			\midrule
            \multicolumn{6}{c}{\textit{Training Parameters}}\\
            \midrule
             Length of trajectory from environment & 20 & 50 & 50 & 100 & 200 \\
             Number of episodes & 1000 & 1000 & 1000 & 1000 & 2500 \\
             \midrule
             \multicolumn{6}{c}{\textit{Evaluation Parameters}}\\
            \midrule
            Horizon $h$ for Entropy metric & 20 & 50 & 50 & 100 & 200 \\
            Horizon $h$ for State Coverage metric & 20 & 50 & 50 & 100 & 200\\
            Horizon $h$ to measure Episode Length & 100 & 500 & 200 & 1000 & 1000 \\
            \bottomrule
    	\end{tabular}
	\caption{Defines the parameters of the environments with discrete actions during training and evaluation, respectively.}
\end{table*}

\begin{table*}[hbt!]
\scriptsize
	\centering
		\begin{tabular}{l|cc}
			\toprule
			\textbf{Name} & \textbf{Reacher} & \textbf{Pusher} \\
			\midrule
            \multicolumn{3}{c}{\textit{Training Parameters}}\\
            \midrule
             Length of trajectory from environment & 100 & 200 \\
             Number of episodes & 1000 & 1000 \\
             \midrule
             \multicolumn{3}{c}{\textit{Evaluation Parameters}}\\
            \midrule
            Horizon $h$ for Entropy metric & 100 & 200 \\
            Horizon $h$ for State Coverage metric & 100 & 200 \\
            \bottomrule
    	\end{tabular}
	\caption{Defines the parameters of the environments with continuous actions during training and evaluation, respectively.}
\end{table*}

\section{Hyper Parameters}
\label{app:hyperparameter}

In this section, we describe the hyperparameters used for training the proposed method \alg{}. \autoref{tab:appendix_hyperparameter_discrete} and \autoref{tab:appendix_hyperparameter_control} presents the list of hyperparameters for the discrete and continuous control environments, respectively.
All models were trained on a single NVIDIA V100 GPU with 32 GB memory.
The implementation of the proposed method was done using the RLHive~\citep{Patil2023} library.

\begin{table*}[hbt!]
	\centering
		\begin{tabular}{l|c}
			\toprule
			\textbf{Name} & \textbf{Value}\\
			\midrule
             Batch Size & 32 \\
             Sequence Length & 10 / 20 / 50 / 50 / 100 \\
             $\alpha$ for $\gamma$-function & 0.95 \\
             Encoder layers & 1\\
             Encoder output dimensions & 64 / 64 / 128 / 128 / 256\\ 
             Encoder activation & LeakyReLU~\cite{yu2017unsupervised} \\
             Hidden state of GRU & 64 / 64 / 128 / 128 / 256\\
             Hidden layer dimension for SR decoder & 32 / 32 / 64 / 64 / 128\\
             Decoder activation & None\\
             Optimizer & Adam~\cite{kingma2014adam} \\
             Learning rate & 3e-4\\
             Capacity of replay buffer & 200000 \\
            \bottomrule
    	\end{tabular}
	\caption{Hyper parameters used for training \alg{}. When parameters are separated by /./././., it means the corresponding hyperparameters for ChainMDP, RiverSwim, Gridworld, TwoRooms and FourRooms environments, respectively. When tuning the agent for a task, we recommend searching over $\alpha \in \{0.9, 0.95, 0.98, 0.99\}$, and hidden state dimension of GRU and encoder output dimensions in $\{64, 128, 256, 512\}$. For the replay buffer, we have used the replay buffer implemented in DreamerV2~\cite{hafner2020dreamerv2}, which for an episode samples a chunk of a given length.}
	\label{tab:appendix_hyperparameter_discrete}
\end{table*}

\begin{table*}[hbt!]
	\centering
		\begin{tabular}{l|c}
			\toprule
			\textbf{Name} & \textbf{Value}\\
			\midrule
             Batch Size & 256 \\
             Sequence Length & 100 \\
             $\alpha$ for $\gamma$-function & 0.95 \\
             Encoder layers & 2\\
             Encoder output dimensions & 256\\ 
             Encoder activation & LeakyReLU~\cite{yu2017unsupervised} \\
             Hidden state of GRU & 256\\
             Hidden layer dimension for SR decoder & 256\\
             Decoder activation & None\\
             Optimizer & Adam~\cite{kingma2014adam} \\
             Learning rate & 3e-4\\
             Capacity of replay buffer & 200000 \\
             Polyak constant & 0.005 \\
             Grad Clip & 5.0 \\
             Action noise & 0.1 \\
             Target noise & 0.2 \\
            \bottomrule
    	\end{tabular}
	\caption{\ Hyper parameters used for training \alg{} on continuous state space environments. When tuning the agent for a task, we recommend searching over $\alpha \in \{0.9, 0.95, 0.98, 0.99\}$, and hidden state dimension of GRU and encoder output dimensions in $\{64, 128, 256, 512\}$. For the replay buffer, we have used the replay buffer implemented in DreamerV2~\cite{hafner2020dreamerv2}, which for an episode samples a chunk of a given length.}
	\label{tab:appendix_hyperparameter_control}
\end{table*}

\section{Meta-RL}
\label{app:metarl}

In this work, we demonstrated how $\eta\psi$-Learning can learn optimal policies that can maximize the entropy of state visitation distribution. Such policies are useful for many subareas of \gls{rl} where during evaluation the task is to adapt to new reward functions with minimal interactions with the environment. 
A challenging subproblem in such tasks is to infer the reward function. This is especially harder when the reward is sparse. 
Some recent works have explored adding exploratory behaviors for initial interactions with the environment to allow agent to infer the reward function. 
VariBAD~\citep{zintgraf2019varibad} algorithm learns optimal policies that can explore during evaluation to speed up adaptation for Meta-\gls{rl} tasks. 
VariBAD maintains a belief over the state space to explore in the environment and upon discovering the reward function adapts to the 
To illustrate the exploratory capabilities of \alg{}, we compare with VariBAD as baseline in this section. 

The implementation of VariBAD provided by the authors is used to conduct this experiment.
The baseline was trained for 10 million environment steps on the $5\times5$ gridworld using the setup described in the paper.
Two metrics are employed to compare the agents:
\begin{itemize}
    \item \textit{State Coverage} computes the fraction of the state space covered by the agent.
    \item \textit{Goal Search Time} computes the environment steps taken to locate the sparse reward goal state. This evaluates the ability of the agent at quickly finding the reward which is essential for swift adaptation to novel tasks.
\end{itemize}
The two metrics evaluate the agents on average time taken to find the reward function and the average search completion time for covering the grid. 
The VariBAD~\cite{zintgraf2019varibad} algorithm considered sparse reward task where a random location is sampled after each episode as the goal state and is assigned a high reward.  
To compute the \textit{Goal Search Time} metric, we sample a goal state randomly and record the steps taken to locate the target state.
For evaluation, the metric is averaged over 16 sampled goal state for each seed.
Figure~\ref{fig:metarl} presents the comparison of VariBAD and \alg{}~on both metrics across 5 seeds.
The proposed method \alg{}~achieves outperforms VariBAD on both metrics while only being trained for 100K environment steps. 
This demonstrate the efficacy of \alg{} at exploring in environment that involves inferring the reward function during evaluation. We believe \alg{} can be combined with a adaptation policy, where the proposed method can explore to find the reward and the adaptation policy is trained to adapt quickly to the reward function, and we leave this for future work.

\section{Ablation Studies}
In this section, we present ablations studies to understand the gains of the proposed method ~\alg{}. 

\subsection{MaxEnt with a recurrent network}
\label{app:maxent_recurrent}
\begin{figure}[ht]
    \centering
    \includegraphics[width=1.0\columnwidth]{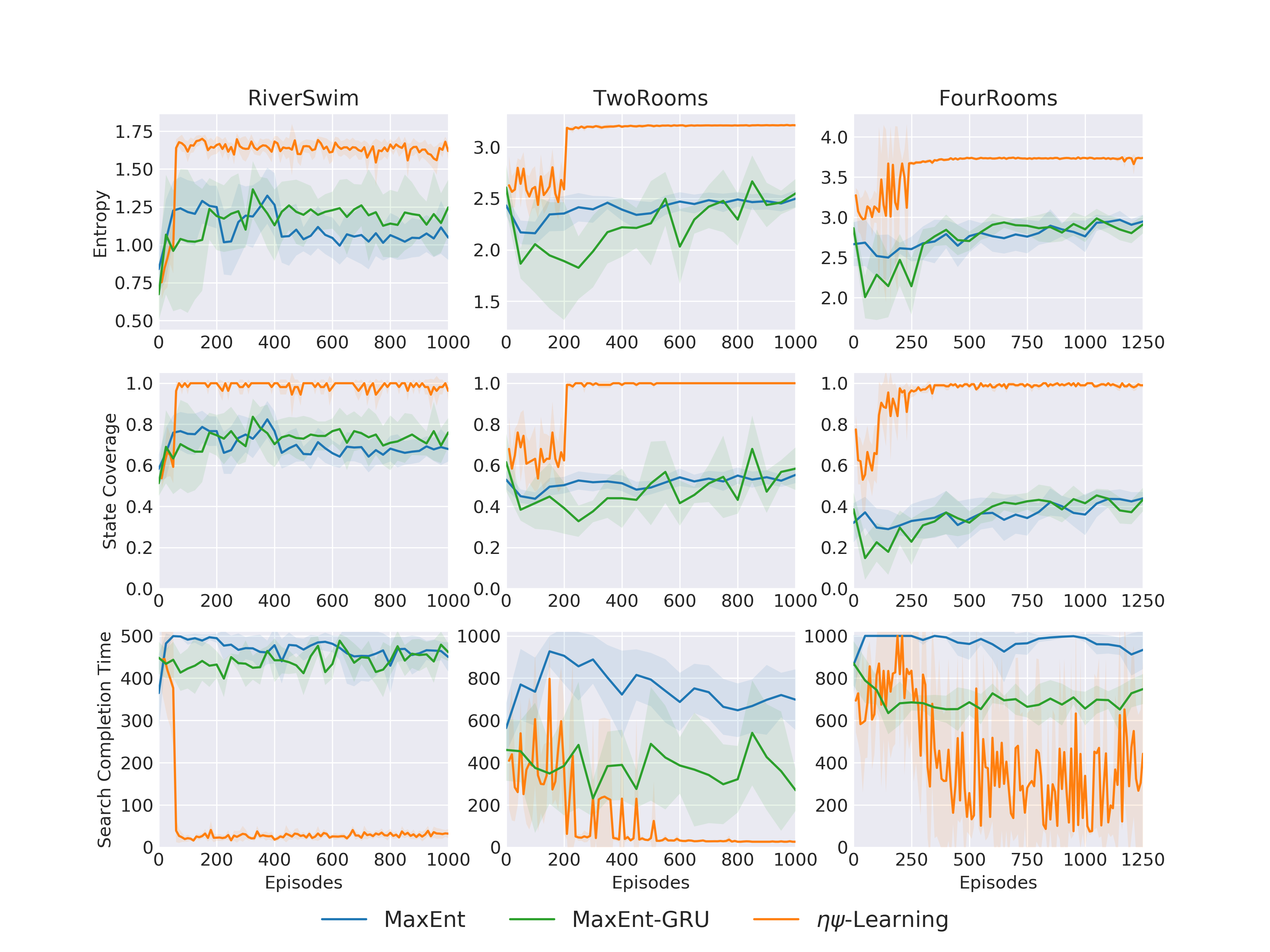}
    \caption{Comparison of the baseline MaxEnt when trained with a recurrent network.}
    \label{fig:ablation_recurrent_maxent}
\end{figure}

We conduct an experiment with a modification to the baseline MaxEnt~\cite{pmlr-v97-hazan19a} where agent observed the history of visited states.
This is done to evaluate if the improvements are coming from having a recurrent policy.
To this end, the state-conditioned policy in MaxEnt is replaced with a recurrent policy where the GRU~\cite{cho2014gru} encodes the states observed in the trajectory. 
The parameters of the recurrent policy is optimized using the loss function described in MaxEnt~\cite{pmlr-v97-hazan19a}. 
\autoref{fig:ablation_recurrent_maxent} presents the results of MaxEnt with a recurrent policy (named MaxEnt-GRU) where no gains are observed by having a recurrent policy and the proposed objective function used to train \alg{} is crucial for learning optimal behaviors.

\subsection{Comparison across multiple trajectories}
\label{app:multiple_trajs}
\begin{figure}
    \centering
    \includegraphics[width=1.0\columnwidth]{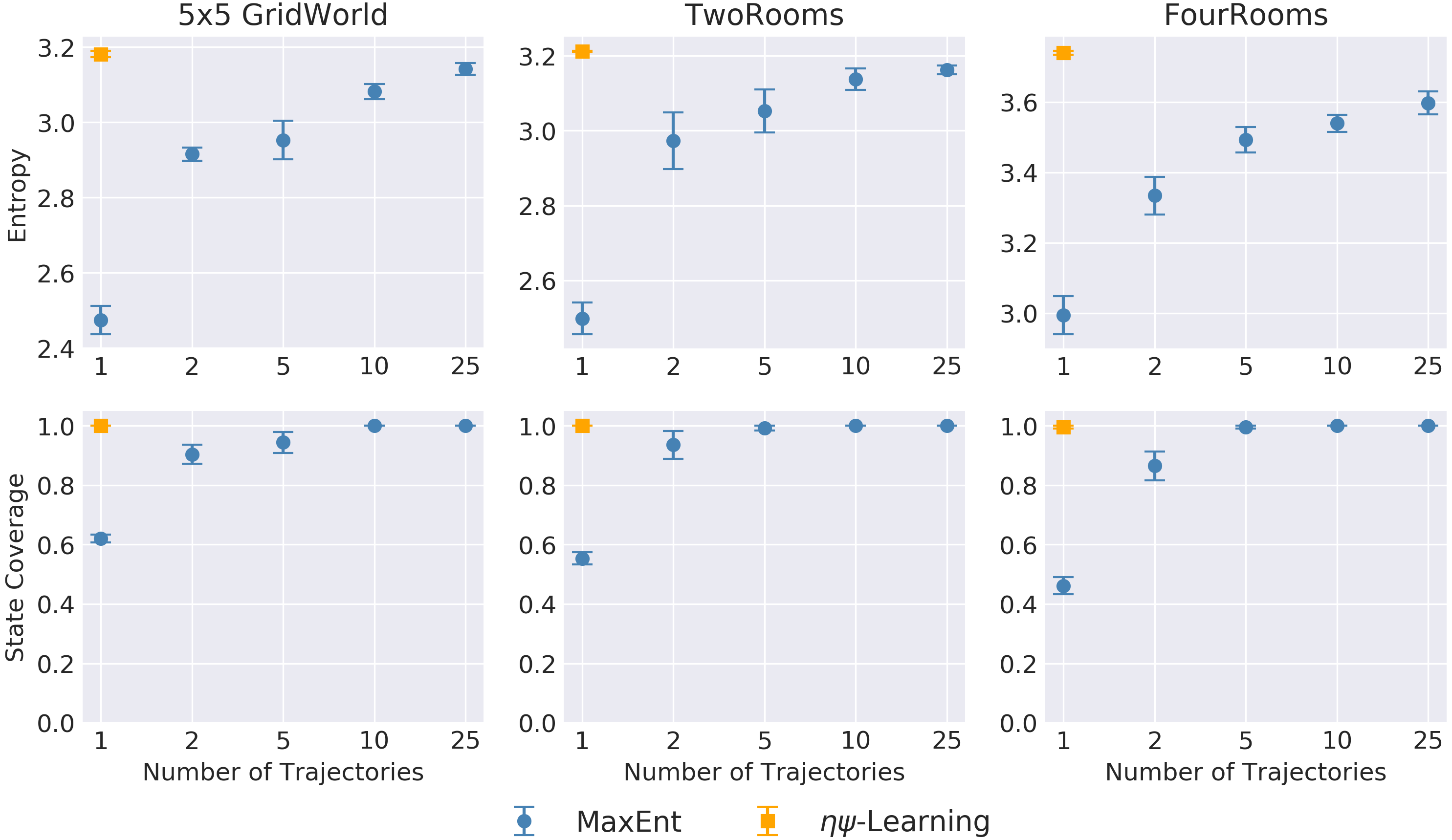}
    \caption{Comparison of \alg{} with baseline MaxEnt when metrics are computed using multiple trajectories. MaxEnt-X denotes the evaluation with X trajectories.}
    \label{fig:ablation_batch}
\end{figure}
Contemporary methods on Maximum State Entropy Exploration~\citep{pmlr-v97-hazan19a,mutti2021task} were evaluated by averaging the state visitation distribution over multiple trajectories.
In this work, we demonstrate that \alg{}~can achieve optimal behaviors over a single trajectory of finite length.
In this study, we also explore comparison of the baseline MaxEnt when evaluated over multiple trajectories. 
For this evaluation, we sample a batch of trajectories and then average the state visitation distribution of trajectories.
The metrics are then computed using this averaged visited state distribution.
In ~\autoref{fig:ablation_batch}, we compare the Entropy and State Coverage over this averaged distribution.
MaxEnt-X denotes the metric of MaxEnt after sampling X trajectories during evaluation.
The proposed method \alg was evaluated using a single trajectory.
We do not report metrics of \alg across multiple trajectories as we observed that the gains do not vanish with an increasing number of trajectories.
The metrics for the MaxEnt algorithm improve with the increasing number of trajectories used for evaluation.
With 10 or more trajectories, the baseline achieves optimal State Coverage.
However, \alg{}~achieves full coverage with a single trajectory demonstrating the efficiency of the exploration policies learned using the proposed method.
The baseline MaxEnt show similar behaviors by improving on the Entropy metric with more trajectories used for evaluation, whereas \alg{}~still outperforms the baseline when evaluated using a single trajectory.
This demonstrates that the proposed method explores the state-space with near-equal state visitations to maximize the entropy while having optimal state coverage in a single trajectory.

\subsection{Effect of the $\alpha$ parameter}
\begin{figure}[ht]
    \centering
    \includegraphics[width=1.0\textwidth]{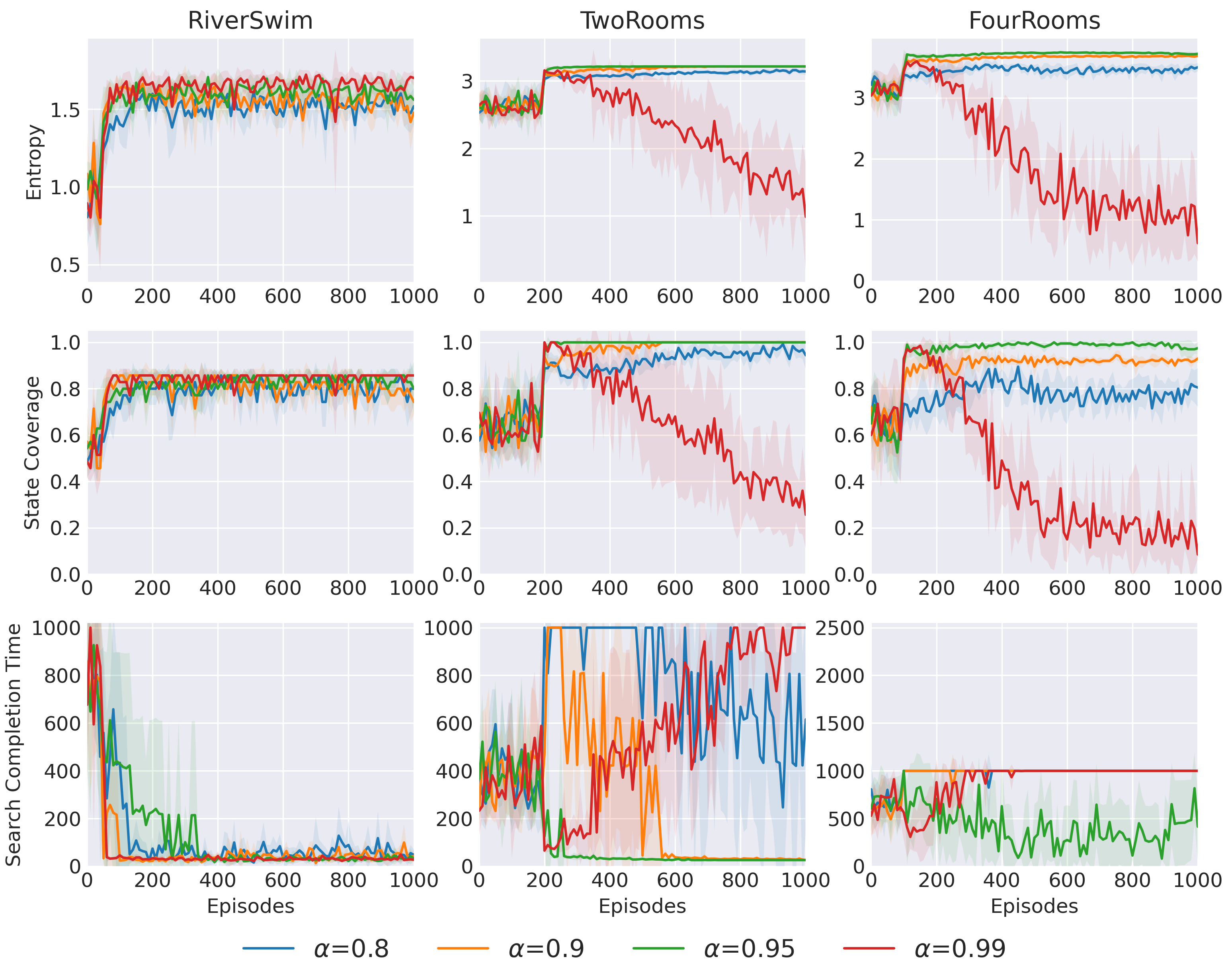}
    \caption{Evaluation with different value of $\alpha$ hyperparameter in the proposed $\gamma$-function.}
    \label{fig:ablation_alpha_2}
\end{figure}

We also study the effect of the hyper-parameter $\alpha$ of the $\gamma$-function (discussed in \autoref{app:discount_function}).
We note that $\alpha$ can be selected using the same method used to select the discount factor in the standard RL. 
We conducted experiments with $\alpha$=\{0.8,0.9,0.95,0.99\} across three environments- RiverSwim, TwoRooms, and FourRooms (\autoref{fig:ablation_alpha_2}). 
On the RiverSwim environment, all methods converged with similar values across all metrics. On the TwoRooms environments, agents with $\alpha$=\{0.8, 0.99\} were not performing well across the three metrics. 
Moreover, the convergence was slower for agent with $\alpha=0.9$ when compared with agent trained with $\alpha=0.95$. 
Our intuition behind this is that when $\alpha$ is smaller, the memory of the visited states in predecessor representation ($\eta$) reduces leading to a re-visitation of observed states. 
Whereas when $\alpha$ is large, then the agent does similarly discount a future state at any point in the trajectory. 
Lastly, on the FourRooms environment, the results are similar to the TwoRooms environment but with more pronounced differences in metrics for different values of $\alpha$. 
This is because FourRooms environment is harder to solve than TwoRooms environments. 
Notably, the Search Completion Time metric diverges for all values of $\alpha \neq 0.95$, which shows that $\alpha=0.95$ leads to optimal behavior for our task. 
For new tasks, we believe $\alpha$ can be tuned similarly to the discount factor used in standard RL.

\subsection{Visualization of trajectories}
\label{app:reacher_pusher_trajs}
The maneuvers taken by the learned \alg{} agent were also visualised. 
For the Reacher environment, the agent was first covering the faraway states followed by covering the nearby states to the central position.
Similar behaviors were observed for the Pusher environment, where the agent was moving the fingertip to different locations on the table top and tried visiting all reachable locations on the table. 
\begin{figure}[ht]
    \centering
    \includegraphics[width=1.0\textwidth]{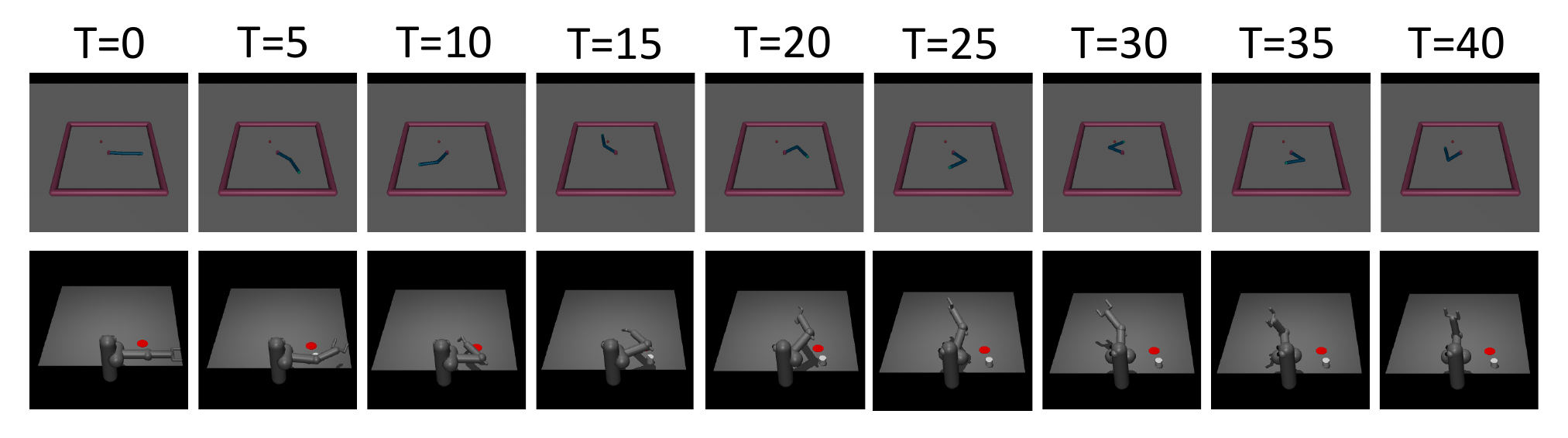}
    \caption{Rolled out trajectories using the learned \alg{} agent on the Reacher (top row) and Pusher (bottom row) environments at different time steps, respectively.}
    \label{fig:ablation_trajectory_reacher_pusher}
\end{figure}

\end{document}